\documentclass[conference]{IEEEtran}
\IEEEoverridecommandlockouts
\usepackage{booktabs}
\usepackage{cite}
\usepackage{url}
\usepackage{tikz}
\usetikzlibrary{positioning, calc}
\usepackage{multirow}
\usepackage{amsmath,amssymb,amsfonts}
\usepackage{amsthm}

\usepackage{algorithmic}
\usepackage{graphicx}
\usepackage{textcomp}
\usepackage{xcolor}
\usepackage{stfloats}
\usepackage{subcaption}
\newcommand{\qPTQ}[2]{\texttt{Q#1\_#2}}
\newcommand{\BF}{\texttt{BF16}\xspace}
\usepackage{enumitem}
\usepackage{csquotes}
\usepackage{hyperref}
\newtheorem{theorem}{Theorem}

\usepackage[table]{xcolor}
\usepackage{xspace}

\usepackage{amsmath}
\usepackage{amssymb}
\usepackage{mathtools}
\usepackage{amsthm}

\usepackage{multirow}

\usepackage[most]{tcolorbox}
\usepackage{tabularx}
\usepackage{makecell}
\usepackage{array}
\newcommand{\eg}{{\it e.g.,~}}

\def\BibTeX{{\rm B\kern-.05em{\sc i\kern-.025em b}\kern-.08em
    T\kern-.1667em\lower.7ex\hbox{E}\kern-.125emX}}
\begin{document}

\title{Accuracy is Not Enough: A Divergence-Based Approach to Evaluate Fidelity Loss in Quantized LLMs}

\author{
\IEEEauthorblockN{
Shahzeb Qamar\textsuperscript{1,2,3,4}, 
Lorenz Sparrenberg\textsuperscript{3,4}, 
Christian Bauckhage\textsuperscript{1,2,3,4}, \\
Baha Rababah\textsuperscript{5}, 
Carson Leung\textsuperscript{5},
Murat Kantarcioglu\textsuperscript{7},
Cuneyt Gurcan Akcora\textsuperscript{6}, 
Rafet Sifa\textsuperscript{1,3,4}
}
\thanks{
This work was funded by the Deutsche Forschungsgemeinschaft (DFG, German Research Foundation) under Germany's Excellence Strategy---EXC 2070 Grant No. 390732324 (Cluster of Excellence PhenoRob) and by the State of North Rhine-Westphalia through the Lamarr Institute for Machine Learning and Artificial Intelligence.
}
\IEEEauthorblockA{
\textsuperscript{1}Fraunhofer IAIS, Germany (\{shahzeb.qamar, rafet.sifa, christian.bauckhage\}[at]iais.fraunhofer.de) \\
\textsuperscript{2}PhenoRob Cluster of Excellence, University of Bonn, Germany \\
\textsuperscript{3}Lamarr Institute for Machine Learning and Artificial Intelligence, Germany \\
\textsuperscript{4}University of Bonn, Germany (lsparren[at]uni-bonn.de) \\
\textsuperscript{5}University of Manitoba, Canada (rababahb[at]myUManitoba.ca, Carson.Leung[at]UManitoba.ca) \\
\textsuperscript{6}University of Central Florida, USA (cuneyt.akcora[at]ucf.edu) \\
\textsuperscript{7}Virginia Tech, USA (muratk[at]vt.edu)
}
}
\maketitle

\begin{abstract}
Deployment of Large Language Models (LLMs) on memory-constrained edge devices relies heavily on aggressive post-training quantization. However, the evaluation of these quantized models is still largely based on zero-shot task accuracy, which depends solely on argmax predictions and is therefore insensitive to changes in the underlying predictive distribution. As a result, accuracy can exhibit unstable and non-monotonic behavior under progressive quantization, masking substantial fidelity loss relative to the BFloat16 (\BF) uncompressed base model and potentially providing misleading signals for deployment decisions.

In this work, we introduce a distribution-sensitive evaluation framework that quantifies information loss in quantized LLMs as the divergence between full-vocabulary predictive distributions at the token decision boundary. Specifically, we compute statistical distances, including Jensen-Shannon Divergence and Total Variation Distance, between the outputs of full-precision and quantized models, enabling a fine-grained analysis of distributional shift.

Using this framework, we quantify probability mass displacement and distributional drift relative to the \BF reference. These measures capture changes in the predictive distribution that may not be reflected in top-1 task accuracy. We conduct a 120-run experimental matrix across five foundation model architectures and four reasoning benchmarks under progressive quantization regimes, from uncompressed \BF to \qPTQ{2}{K}, providing a systematic fidelity analysis.

Our results show that divergence metrics generally increase under stronger quantization in the evaluated settings, complementing task accuracy with a fidelity signal relative to the \BF reference. Across the tested llama.cpp quantization schemes, mixed-precision \qPTQ{4}{K} generally yields lower divergence than uniform \qPTQ{4}{0} at similar memory footprints. These findings motivate distribution-aware evaluation as a practical diagnostic complement to task accuracy; they do not directly establish correctness, calibration, safety, or user-perceived quality.
\end{abstract}
\begin{figure}[t]
    \centering
    \includegraphics[width=\columnwidth]{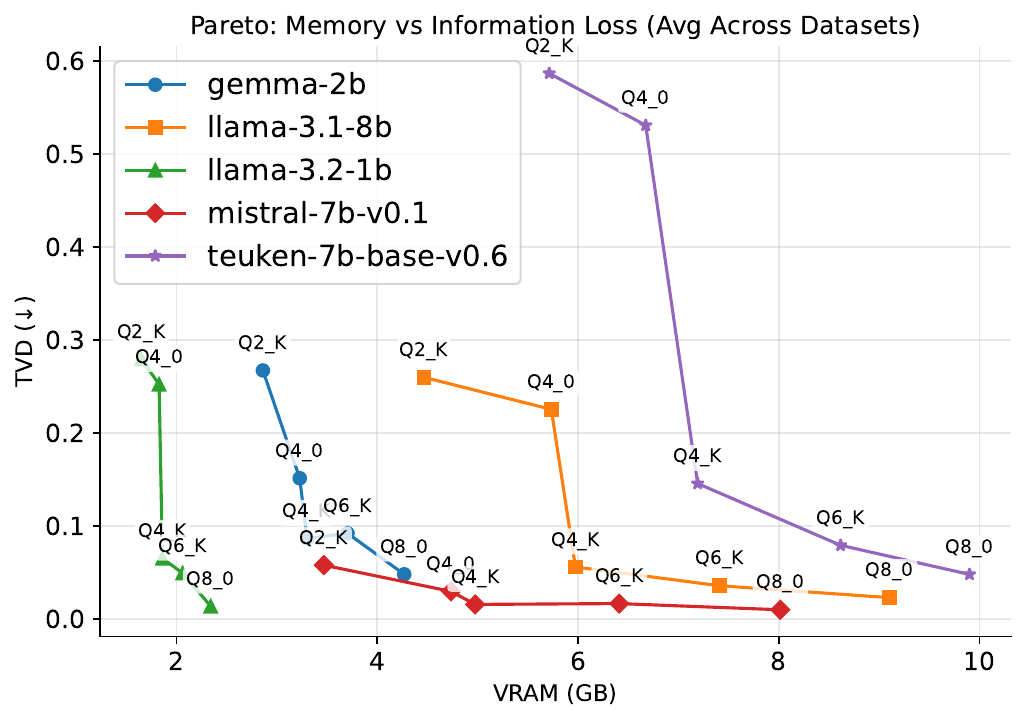} %
    \caption{VRAM requirement versus average TVD in the evaluated settings. The plot illustrates that more aggressive compression is associated with increased distributional divergence from the \BF reference, with \qPTQ{4}{K} providing a favorable empirical trade-off for the tested models and quantization schemes.}
    \label{fig:pareto_frontier}
\end{figure}
\section{Introduction}
\label{intro}
The rapid scaling of Large Language Models (LLMs) is fundamentally bottlenecked by the strict memory constraints of edge devices \cite{team2024llama, jiang2023mistral7b, gemmateam2024gemma2improvingopen}. To bridge this gap, Post-Training Quantization (PTQ) serves as the standard for model compression \cite{frantar2023gptq, lin2024awq, xiao2023smoothquant, chee2023quip}, with frameworks like \texttt{llama.cpp} \cite{gerganov2023llama} popularizing mixed-precision schemes (e.g., K-quants) \cite{sparrenberg2025small}. Despite these advancements, standard evaluation methodologies remain flawed~\cite{gong2024llm}. Industry benchmarks heavily rely on zero-shot task accuracy (i.e., evaluation without task-specific demonstrations) (\eg MMLU~\cite{hendrycks2020measuring}, BoolQ~\cite{clark2019boolq}, PIQA~\cite{bisk2020piqa}, HELM~\cite{liang2023holistic}). Because accuracy evaluates only whether the predicted \textit{argmax} token matches the ground truth, it reduces model behavior to a coarse binary outcome and discards the underlying predictive distribution. Consequently, a degraded low-bit model might guess the correct token purely by chance from a flattened distribution, or default to a dominant prior token, scoring artificially high without true task understanding. This masking effect creates a misleading signal of model reliability~\cite{huang2024good}.

Building on Dutta et al.~\cite{dutta2024accuracy}, who showed that stable aggregate accuracy can hide severe internal prediction \textquote{flips}~\cite{hooker2019compressed}, we argue that quantized LLM evaluation must move from discrete \textit{argmax} outcomes to continuous, full-vocabulary probability distributions. Standard benchmarks often treat LLMs as deterministic classifiers, yet the \textit{argmax} operation discards most of the model's uncertainty information and can score a confident \BF model and a degraded quantized model as equivalent~\cite{schaeffer2023emergent, kadavath2022language}. This concern is consistent with evidence from computer vision~\cite{guo2017calibration}, adversarial robustness~\cite{szegedy2014intriguing}, and medical AI~\cite{begoli2019need}, where accuracy alone has proven insufficient for reliability assessment. Foundational LLM work has likewise shown that confidence and calibration are necessary for reliable deployment~\cite{kadavath2022language}.
\begin{table*}[t]
\centering
\caption{Comparison of the Proposed Framework with Recent LLM Evaluation Literature}
\label{tab:related_works_comparison}
\renewcommand{\arraystretch}{1.2}
\begin{tabular}{@{}l l c c l@{}}
\toprule
\textbf{Work} & \textbf{Evaluation Focus} & \textbf{Distributional Analysis?} & \textbf{Quantization Focus} & \textbf{Primary Metric(s)} \\
\midrule
Gong et al. (2023) \cite{gong2024llm} & General PTQ Benchmarking & No (Top-1 Argmax) & Uniform & Zero-Shot Accuracy \\
Huang et al. (2024) \cite{huang2024good} & Low-bit Degradation Trends & No (Top-1 Argmax) & Uniform & Accuracy \& Perplexity \\
Dutta et al. (2024) \cite{dutta2024accuracy} & Prediction Instability & Partial (KL-Divergence) & Uniform & Flip Rate \& KL-Div. \\
Kübler et al. (2026) \cite{kubler2026llms} & Statistical Significance & No (Binary Sample-Level) & Uniform \& Mixed & McNemar's Test \\
\midrule
\textbf{Proposed Framework} & \textbf{Structural Distributional Drift} & \textbf{Yes (Full Vocabulary)} & \textbf{Uniform \& Mixed} & \textbf{TVD \& JSD} \\
\bottomrule
\end{tabular}
\end{table*}

Drawing on these cross-disciplinary principles, we investigate whether aggressive quantization induces distributional drift relative to the \BF reference model even when task accuracy remains stable. To quantify this drift, we use Total Variation Distance (TVD)~\cite{cover2006elements} and Jensen-Shannon Divergence (JSD)~\cite{lin1991divergence}, computed over the full vocabulary. These metrics measure reference-distribution fidelity; they do not by themselves establish correctness, calibration, safety, or downstream utility. 

Moving beyond binary accuracy, our approach provides distribution-level fidelity signals under stronger compression (Figure~\ref{fig:pareto_frontier}). Our four main contributions are: (1) we show that task accuracy can mask substantial distributional differences relative to a \BF reference; (2) we introduce a full-vocabulary fidelity analysis using bounded TVD and JSD, supported by a coarse-graining result; (3) we evaluate $120$ configurations across five foundation models, four benchmarks, and six quantization regimes; and (4) for the tested llama.cpp schemes, we find that mixed-precision \qPTQ{4}{K} generally yields lower divergence than uniform \qPTQ{4}{0} at similar memory footprints.

\section{Related Work}
\subsection{LLM Compression and the Precision Bottleneck}
To mitigate the latency and privacy bottlenecks of cloud-centric LLMs~\cite{brown2020language, llama2023llama, carlini2021extracting, mireshghallah2020privacy}, Post-Training Quantization (PTQ) has become a standard approach for deploying large models on VRAM-constrained edge devices~\cite{chen2026device, sparrenberg2025small}. Unlike Quantization-Aware Training, PTQ avoids costly retraining while reducing model memory through lower-precision weights and activations \cite{zhu2024survey, nagel2020up}. This massive parameter reduction facilitates viable on-device local inference for latency-sensitive and privacy-critical applications, such as real-time critical error detection in machine translation \cite{chopra2025small}. While PTQ methods reduce memory cost enough to make edge deployment feasible~\cite{frantar2023gptq, lin2024awq}, they do not guarantee that the quantized model preserves the full predictive distribution of its \BF reference. Recently, open-source deployment frameworks have popularized mixed-precision schemes such as K-quants. Unlike legacy uniform quantization (\eg \qPTQ{4}{0}), which applies the same low-bit format across layers, mixed-precision quantization preserves higher precision (\eg 6-bit) in sensitive components while compressing more robust feed-forward layers more aggressively~\cite{sparrenberg2025small}. These methods address the hardware bottleneck, but evaluating whether quantized models retain predictive fidelity remains unresolved.

\subsection{The Illusion of Accuracy and Distributional Divergence}
The prevailing paradigm for evaluating quantized LLMs relies on discrete, zero-shot accuracy \cite{gong2024llm}. However, top-1 accuracy can create statistical mirages that obfuscate continuous model behaviors \cite{schaeffer2023emergent}. Huang et al. \cite{huang2024good} observed that low-bit models exhibit erratic, non-monotonic behavior on reasoning tasks. Critically, Dutta et al. \cite{dutta2024accuracy} highlighted that aggregate accuracy frequently masks massive internal prediction instability, leading to \textquote{flips} in previously correct answers. 

While Dutta et al.~\cite{dutta2024accuracy} exposed this instability through discrete prediction flips and KL divergence, their distributional analysis inherits KL divergence's sensitivity to near-zero probabilities, which are common in extreme low-bit quantization. Similarly, Kübler et al.~\cite{kubler2026llms} rigorously quantified model degradation using McNemar's test, and Rababah et~al.~\cite{rababah2026illusion} characterized behavioral divergences through sample-level correctness agreement. However, these analyses remain fundamentally tied to binary correct/incorrect outcomes rather than the continuous underlying predictive structure.
 
Because accuracy depends only on the model's top (\textit{argmax}) prediction, it disregards the full probability distribution, frequently masking significant behavioral shifts occurring beneath the surface. This structural limitation is extensively documented in broader deep learning literature. For instance, Guo et al.~\cite{guo2017calibration} demonstrated that modern neural networks can maintain high task accuracy while becoming severely uncalibrated. Similarly, adversarial robustness literature reveals that model vulnerabilities are often detected not by immediate accuracy drops, but by measuring divergence in the underlying softmax distributions \cite{szegedy2014intriguing}. Furthermore, Kadavath et al.~\cite{kadavath2022language} recently emphasized that evaluating discrete accuracy without assessing an LLM's internal confidence calibration leads to brittle, unreliable deployments. This vulnerability extends directly to model compression; for instance, Xu et al.~\cite{xu2026alignment} recently demonstrated that quantization can silently destroy a model's safety alignment while leaving traditional accuracy and perplexity metrics effectively unchanged. Recognizing this, recent efforts have actively begun exploring uncertainty-aware and calibration-focused evaluations tailored specifically for low-bit quantized LLMs \cite{sparrenberg2025towards}.  Consequently, ensuring the reliability of edge-deployed compressed models necessitates moving beyond binary outcomes.  
To address the blind spots of discrete task-level evaluation, we introduce a continuous framework to quantify probability mass displacement relative to the \BF reference. This analysis complements discrete task metrics but does not define a universal deployment threshold. 

\section{Background and Preliminaries}
\subsection{Uniform vs. Mixed-Precision Quantization}
In post-training quantization, the method of bit-allocation drastically impacts the structural preservation of the model. Traditional legacy approaches (\eg \qPTQ{4}{0}) utilize \textit{uniform quantization}, wherein every weight tensor across the entire network is rigidly compressed to the same bit-width (\eg 4-bit). While computationally simple, uniform quantization often destroys highly sensitive outlier weights and degrades critical attention mechanisms \cite{lin2024awq, xiao2023smoothquant}. 
\begin{figure*}
    \centering
    \includegraphics[width=0.8\linewidth]{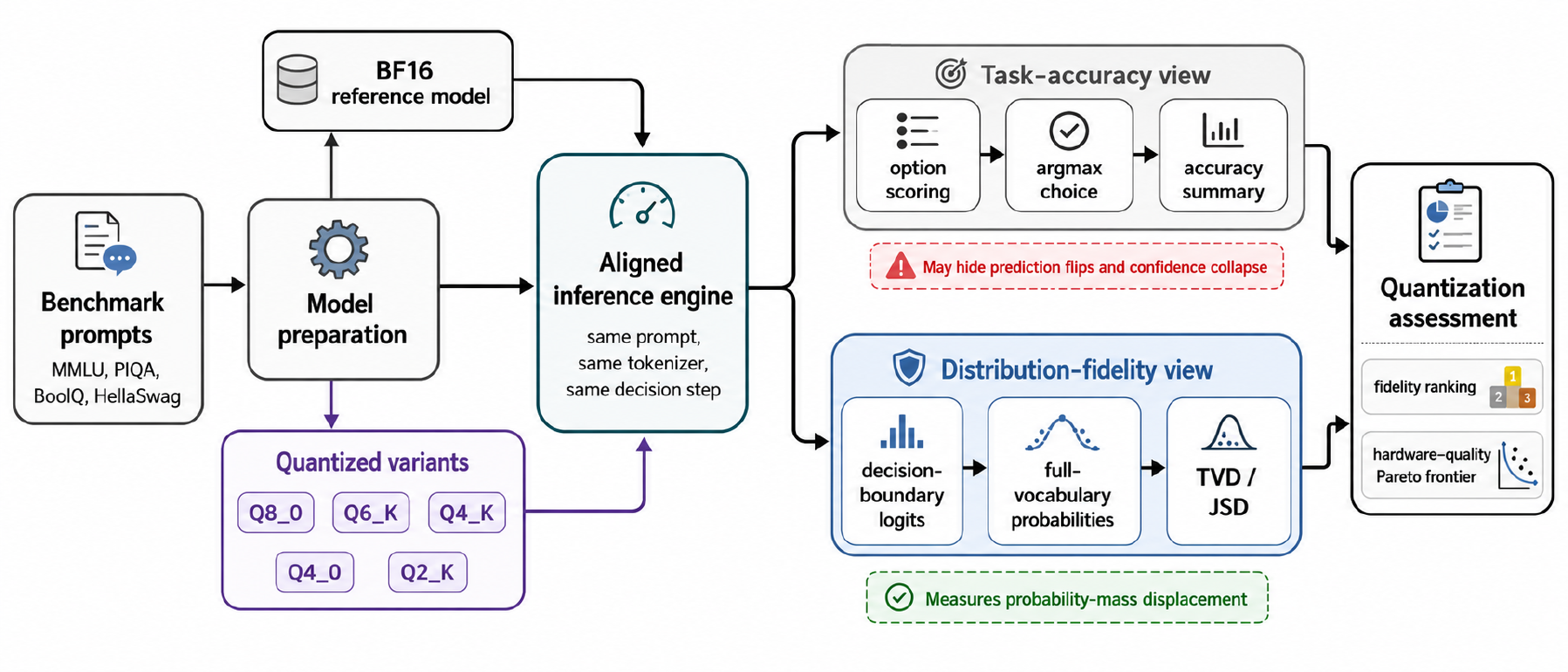}
   \caption{Evaluation framework for quantized LLMs. The pipeline compares \BF and quantized models through aligned inference, contrasting standard task accuracy with distribution-level fidelity metrics such as TVD and JSD.}
    \label{fig:approach}
\end{figure*}
Conversely, modern \texttt{llama.cpp} implementations utilize K-quantization (\eg \qPTQ{4}{K}), a sophisticated \textit{mixed-precision} heuristic \cite{sparrenberg2025small}. To fully understand its advantage, it is essential to contrast it with legacy uniform quantization. Legacy methods rely on basic block quantization, assigning a full-precision (\BF)
scale and offset to every small block of weights (\eg 32 parameters), which introduces significant memory bandwidth overhead. In contrast, K-quants implement a highly efficient two-step ``double quantization'' architecture. They group multiple standard blocks together into a larger ``super-block.'' The scaling factors for the individual constituent blocks are themselves quantized down to 8-bit integers, and a single set of FP16 super-scales is stored for the entire super-block to dequantize them. This two-level scheme slashes metadata overhead and heavily optimizes CPU cache line reads.

Furthermore, rather than enforcing a blanket bit-width across the network, K-quants profile the architecture to distribute precision dynamically. They assign higher precision (\eg 6-bit) to mathematically sensitive structures, such as attention layers and final output tensors, while aggressively compressing the robust, over-parameterized feed-forward networks down to 3-bit or 4-bit. 

\subsection{Full-Vocabulary Probability Distribution}

Traditional task-accuracy metrics evaluate LLMs primarily based on whether the correct candidate option is selected, often using scores derived from the Negative Log-Likelihood (NLL) over a small subset of candidate answers. In contrast, to quantify structural information loss, we evaluate the model's predictive uncertainty across its entire vocabulary $\mathcal{V}$ at the immediate decision boundary.

Let $f_{\theta}$ denote the uncompressed reference LLM with parameters $\theta$, and let $f_{\hat{\theta}}$ denote its quantized counterpart obtained through post-training quantization. Both models share the same tokenizer and vocabulary $\mathcal{V}$, differing only in parameter precision. Given an input prompt sequence $X$, both models produce next-token logit vectors, $\mathbf{z} \in \mathbb{R}^{|\mathcal{V}|}$, over $\mathcal{V}$. 

This logit vector is mapped to a conditional probability distribution using the softmax function:

$$ p(v_i \mid X) = \frac{\exp(z_i)}{\sum_{j=1}^{|\mathcal{V}|} \exp(z_j)} $$

where $p(v_i \mid X)$ represents the model's scalar probability (confidence) for the $i$-th token $v_i \in \mathcal{V}$, conditioned explicitly on the input prompt $X$. Specifically, let $P$ denote the full predictive distribution produced by the uncompressed base model $f_{\theta}$, and let $Q$ denote the corresponding predictive distribution produced by the quantized model $f_{\hat{\theta}}$.

\subsection{Statistical Divergence Metrics}

To measure probability mass displacement, where quantization noise redistributes probability mass away from the uncompressed $f_{\theta}$ model's reference distribution, we employ two distinct statistical divergence metrics: Jensen-Shannon Divergence (JSD) and Total Variation Distance (TVD). 

We select TVD and JSD because both are bounded and symmetric, whereas KL divergence is unbounded and can be dominated by low-probability mismatches. TVD provides a direct probability-mass interpretation, while JSD provides a bounded information-theoretic comparison \cite{guo2017calibration, szegedy2014intriguing}. We use these metrics to quantify distributional drift relative to the \BF reference that may not be visible through zero-shot accuracy alone.

\textbf{Jensen-Shannon Divergence (JSD)}~\cite{lin1991divergence}: To capture information-theoretic divergence stably, JSD evaluates the divergence of both distributions from their smoothed average mixture $M = \frac{1}{2}(P + Q)$:
$$ JSD(P \parallel Q) = \frac{1}{2} D_{KL}(P \parallel M) + \frac{1}{2} D_{KL}(Q \parallel M) $$
where $D_{KL}(\cdot \parallel \cdot)$ denotes the standard Kullback-Leibler divergence. When computed using the base-2 logarithm, JSD is strictly bounded within the interval $[0, 1]$. In the context of our evaluation, a JSD of $0$ indicates that the quantized model's predictive distribution perfectly matches the uncompressed baseline (zero fidelity loss).
Conversely, a JSD approaching $1$ indicates nearly disjoint predictive distributions and therefore severe divergence from the \BF reference.

\textbf{Total Variation Distance (TVD)}~\cite{cover2006elements}: While JSD provides an information-theoretic comparison, TVD measures the absolute difference between distributions. TVD equals the minimum fraction of probability mass that must be reassigned to transform the base distribution $P$ into the quantized distribution $Q$. For discrete probability distributions over a vocabulary $\mathcal{V}$, TVD is defined mathematically as half the $L_1$ distance:

\begin{equation}
\text{TVD}(P, Q) = \frac{1}{2} \sum_{v \in \mathcal{V}} |P(v) - Q(v)| 
\end{equation}

Because TVD is strictly bounded within $[0, 1]$, it allows for direct physical interpretation. 
For example, a TVD of $0.25$ means that $25\%$ of probability mass must be reassigned to transform one distribution into the other. Together, these metrics quantify fidelity loss relative to the \BF reference caused by quantization.
\section{Methodology and Experimental Setup}
\subsection{Design Rationale and Dual-Stream Pipeline}
To operationalize our information-theoretic framework, we must physically capture the probability mass displaced by quantization noise. Because standard accuracy metrics rely on a lossy, non-invertible \textit{argmax} function, they are blind to shifts in the underlying uncertainty manifold. We compute statistical divergence between full-vocabulary softmax distributions to quantify fidelity to the \BF reference that zero-shot accuracy may conceal. 
To rigorously quantify this phenomenon, we design a dual-stream evaluation pipeline. Stream 1 computes zero-shot task accuracy by leveraging native Key-Value caching and selecting the candidate option that minimizes the standard length-normalized negative log-likelihood (NLL) utilized by established benchmarks~\cite{liang2023holistic}:
$\text{NLL}(Y) = - \frac{1}{|Y|} \sum_{i=1}^{|Y|} \log P(y_i \mid X, y_{<i})$.

For each prompt, Stream 2 evaluates the prompt, extracts the backend-provided full-vocabulary output vector, and applies a temperature-$1.0$ softmax without distribution truncation. We then compare the \BF and quantized distributions using TVD and JSD. Thus, Stream 2 is a prompt-conditioned full-vocabulary fidelity diagnostic. For multi-token answer options, NLL scores the complete candidate sequence, whereas TVD and JSD characterize the extracted prompt-level distribution; the two streams are therefore complementary rather than equivalent task measures. Answer-level aggregation maps the vocabulary to a smaller output space, so the following theorem formalizes that such coarse-graining cannot increase distributional discrepancy.

\begin{theorem}
Let $K(a \mid v)$ be a stochastic coarse-graining operator mapping a
full-vocabulary distribution over $\mathcal{V}$ to an answer-level
distribution over $\mathcal{A}$. Then, for all
$P, Q \in \Delta(\mathcal{V})$,
\begin{align}
\mathrm{TVD}(KP, KQ) &\leq \mathrm{TVD}(P, Q), \nonumber\\
\mathrm{JSD}(KP, KQ) &\leq \mathrm{JSD}(P, Q).
\end{align}
\end{theorem}

\begin{proof} For TVD, we use $ TVD(P,Q)=\frac12\sum_v |P(v)-Q(v)|. $ Then \begin{align*} 2TVD(KP,KQ) &= \sum_a | \sum_v K(a| v)(P(v)-Q(v)) | \\ &\le \sum_a \sum_v K(a| v)|P(v)-Q(v)| \\ &= \sum_v |P(v)-Q(v)| \sum_a K(a| v) \\ &= \sum_v |P(v)-Q(v)| = 2\,TVD(P,Q)\end{align*} since $ \sum_a K(a| v)=1 $ for every $v$. Dividing by $2$ gives $ TVD(KP,KQ)\le TVD(P,Q).$ We give the proof for JSD in App.~\ref{sec:theorems}  
\end{proof}

\subsection{Quantization, Prompting, and Flip Analysis}
To guarantee that measured divergence stems strictly from precision truncation, all quantized models are  generated directly from their \BF Hugging Face baselines using \texttt{llama.cpp}. To objectively map the degradation curve, we evaluate five distinct compression states: 8-bit uniform (\qPTQ{8}{0}), 6-bit mixed (\qPTQ{6}{K}), 4-bit uniform (\qPTQ{4}{0}), 4-bit mixed (\qPTQ{4}{K}), and 2-bit mixed (\qPTQ{2}{K}). To control prompt-template variation, we enforce fixed continuous-generation formatting templates (\eg \texttt{Question: \{question\} \textbackslash nAnswer:}) to bypass instruction-tuning chats. Furthermore, adapting the methodology of Dutta et al. \cite{dutta2024accuracy}, we track discrete prediction ``total flips'' (the aggregate sum of both \textit{Flip-to-Wrong} and \textit{Flip-to-Correct} transitions) between the \BF and quantized states to expose the overall instability of task-level accuracy.
\subsection{Metrics} 
We use TVD and JSD as fidelity diagnostics relative to the \BF reference and compare their trends with task accuracy and prediction flips. Our objective is to determine whether distributional divergence distinguishes quantization schemes that appear similar under conventional task metrics. While both TVD and JSD are continuously tracked (as illustrated in Figure~\ref{fig:4panel_divergence}), we prioritize TVD as the primary reporting metric in our tabular results because it has a direct interpretation as the fraction of probability mass that would need to be reassigned to transform one predictive distribution into the other.

\subsection{Evaluated Models, Datasets, and Hardware}
We evaluate five diverse foundation models targeted for edge deployment: LLaMA-3.1-8B \cite{team2024llama}, LLaMA-3.2-1B, Mistral-7B-v0.1 \cite{jiang2023mistral7b}, Teuken-7B-base-v0.6 \cite{ali2025teuken}, and Gemma-2B \cite{gemmateam2024gemma2improvingopen}. To ensure domain diversity, we utilize four established zero-shot benchmarks: MMLU, BoolQ, PIQA, and CausalBench, with full dataset statistics detailed in Table~\ref{tab:dataset_stats}. Inference is executed deterministically (\texttt{seed=42}) on a single NVIDIA V100 32GB GPU, constrained to a 2048 context window, with dynamic VRAM allocation tracked via standard driver queries.

\begin{table}[t]
\centering
\caption{
Evaluation benchmarks used in the zero-shot experiments.}
\label{tab:dataset_stats}
\footnotesize
\setlength{\tabcolsep}{4pt}
\renewcommand{\arraystretch}{1.1}
\begin{tabular}{@{}l r l l@{}}
\toprule
\textbf{Dataset} & \textbf{Entries} & \textbf{Task Type} & \textbf{Split} \\
\midrule
MMLU~\cite{hendrycks2020measuring} & 1,530 & Multidomain QA & Validation \\
BoolQ~\cite{clark2019boolq} & 3,270 & Reading comprehension & Validation \\
PIQA~\cite{bisk2020piqa} & 1,838 & Physical commonsense & Validation \\
CausalBench~\cite{wang2024causalbench} & 1,500 & Causal reasoning & Stratified subset \\
\bottomrule
\end{tabular}
\vspace{-5px}
\end{table}
\section{Results and Discussion}

\subsection{The Masking Effect of Task Accuracy}

\begin{table*}[t]
\centering
\caption{Impact of Quantization on Accuracy, Prediction Stability, and Information Loss Across Reasoning (PIQA) and Knowledge (MMLU) Domains.}
\label{tab:main_results_comprehensive}
\renewcommand{\arraystretch}{1.15}
\resizebox{\linewidth}{!}{
\begin{tabular}{@{}l|r|ccc|ccc@{}}
\toprule
\multirow{2}{*}{\textbf{Model \& Quantization}} & \multirow{2}{*}{\textbf{VRAM}} & \multicolumn{3}{c|}{\textbf{PIQA (Language Reasoning)}} & \multicolumn{3}{c}{\textbf{MMLU (Factual Knowledge)}} \\ \cmidrule(l){3-8} 
 &  & \textbf{Accuracy($\uparrow$)} & \textbf{\% Flips from \BF} & \textbf{TVD($\downarrow$)} & \textbf{Accuracy($\uparrow$)} & \textbf{\% Flips from \BF} & \textbf{TVD($\downarrow$)} \\ \midrule

Teuken-7B-base-v0.6 (\BF) & $14.8$ GB & $68.9\%$ & --- & --- & $\mathbf{31.1\%}$ & --- & --- \\
Teuken-7B-base-v0.6 \qPTQ{4}{K} (Mixed) & $6.6$ GB & $\mathbf{69.2\%}$ & $\mathbf{3.43\%}$  & $\mathbf{0.136}$ & $30.3\%$ & $\mathbf{3.92\%}$ & $\mathbf{0.143}$ \\
Teuken-7B-base-v0.6 \qPTQ{4}{0} (Uniform) & $6.1$ GB & $68.7\%$ & $4.84\%$  & $0.560$ & $30.3\%$ & $4.58\%$ & $0.541$ \\ \midrule

LLaMA-3.1-8B (\BF) & $14.8$ GB & $\mathbf{77.3\%}$ & --- & --- & $32.1\%$ & --- & --- \\
LLaMA-3.1-8B \qPTQ{4}{K} (Mixed) & $5.7$ GB & $77.1\%$ & $\mathbf{2.72\%}$  & $\mathbf{0.052}$ & $\mathbf{32.2\%}$ & $\mathbf{2.48\%}$ & $\mathbf{0.055}$ \\
LLaMA-3.1-8B \qPTQ{4}{0} (Uniform) & $5.5$ GB & $77.2\%$ & $4.08\%$  & $0.215$ & $32.1\%$ & $2.94\%$ & $0.221$ \\ \midrule

Mistral-7B-v0.1 (\BF) & $14.0$ GB & $\mathbf{78.8\%}$ & --- & --- & $\mathbf{33.5\%}$ & --- & --- \\
Mistral-7B-v0.1 \qPTQ{4}{K} (Mixed) & $4.9$ GB & $77.6\%$ & $\mathbf{3.26\%}$  & $\mathbf{0.009}$ & $32.8\%$ & $\mathbf{3.46\%}$ & $\mathbf{0.015}$ \\
Mistral-7B-v0.1 \qPTQ{4}{0} (Uniform) & $4.6$ GB & $76.9\%$ & $4.08\%$  & $0.030$ & $31.9\%$ & $6.08\%$ & $0.030$ \\ 
\bottomrule
\end{tabular}}
\end{table*}

A central hypothesis of this study is that task-level accuracy metrics are insufficient for characterizing fidelity to the \BF reference in heavily quantized models. Our empirical results in Table~\ref{tab:main_results_comprehensive} immediately reveal a critical masking effect, wherein task accuracy remains artificially stable despite severe degradation in the model's internal predictive behavior. For example, when observing the Teuken-7B-base-v0.6 model compressed to 4-bit uniform precision (\qPTQ{4}{0}) on the PIQA dataset, zero-shot accuracy drops by a negligible $0.2\%$ (from $68.9\%$ to $68.7\%$). However, the same setting yields $4.84\%$ prediction flips and a TVD of $0.560$, meaning that transforming its next-token distribution into the \BF reference distribution requires reassigning $56\%$ of probability mass. This substantial reference-distribution divergence is not reflected by raw accuracy alone.

This phenomenon, and the danger of relying solely on accuracy, is most pronounced in the 1B and 2B parameter architectures (\textit{Gemma-2B} and \textit{LLaMA-3.2-1B}). As detailed in App. Table \ref{tab:app_causalbench}, when evaluated on the \textit{CausalBench} dataset, both models achieve an accuracy of approximately $50.1\%$ across nearly all quantization thresholds, from the uncompressed \BF baseline down to 4-bit mixed precision (\qPTQ{4}{K}). Evaluated in a vacuum, this metric suggests stable, albeit baseline, performance. However, the confusion matrix reveals Precision and Recall scores of exactly $0.0$ with zero prediction flips. The models predict the negative class uniformly in this evaluation setup. While these discrete metrics expose this degenerate prediction pattern, they do not quantify the associated divergence from the \BF predictive distribution.

A similar majority-class prediction pattern occurs on the \textit{BoolQ} benchmark (see App. Table \ref{tab:app_boolq}). Across all evaluated quantization levels (down to \qPTQ{2}{K}), both the 1B and 2B models achieve an exact accuracy of $62.2\%$ with zero prediction flips. While this successfully matches the majority class distribution of the validation split, both models yield a Precision of $0.6217$ and a Recall of exactly $1.0$. This indicates that these models predict the \textquote{True} token for all evaluated examples. This behavior is consistent with majority-class prediction, although the present analysis does not establish its cause.

Because these models produce the same class prediction across quantization states, prediction flips can remain zero despite distributional differences. In this setting, top-1 accuracy alone does not reveal the underlying prediction pattern. Precision, recall, and TVD provide complementary information: the former characterize class-level behavior, while TVD quantifies divergence from the \BF next-token distribution.

\begin{figure*}[t]
    \centering
    \includegraphics[width=0.8\textwidth]{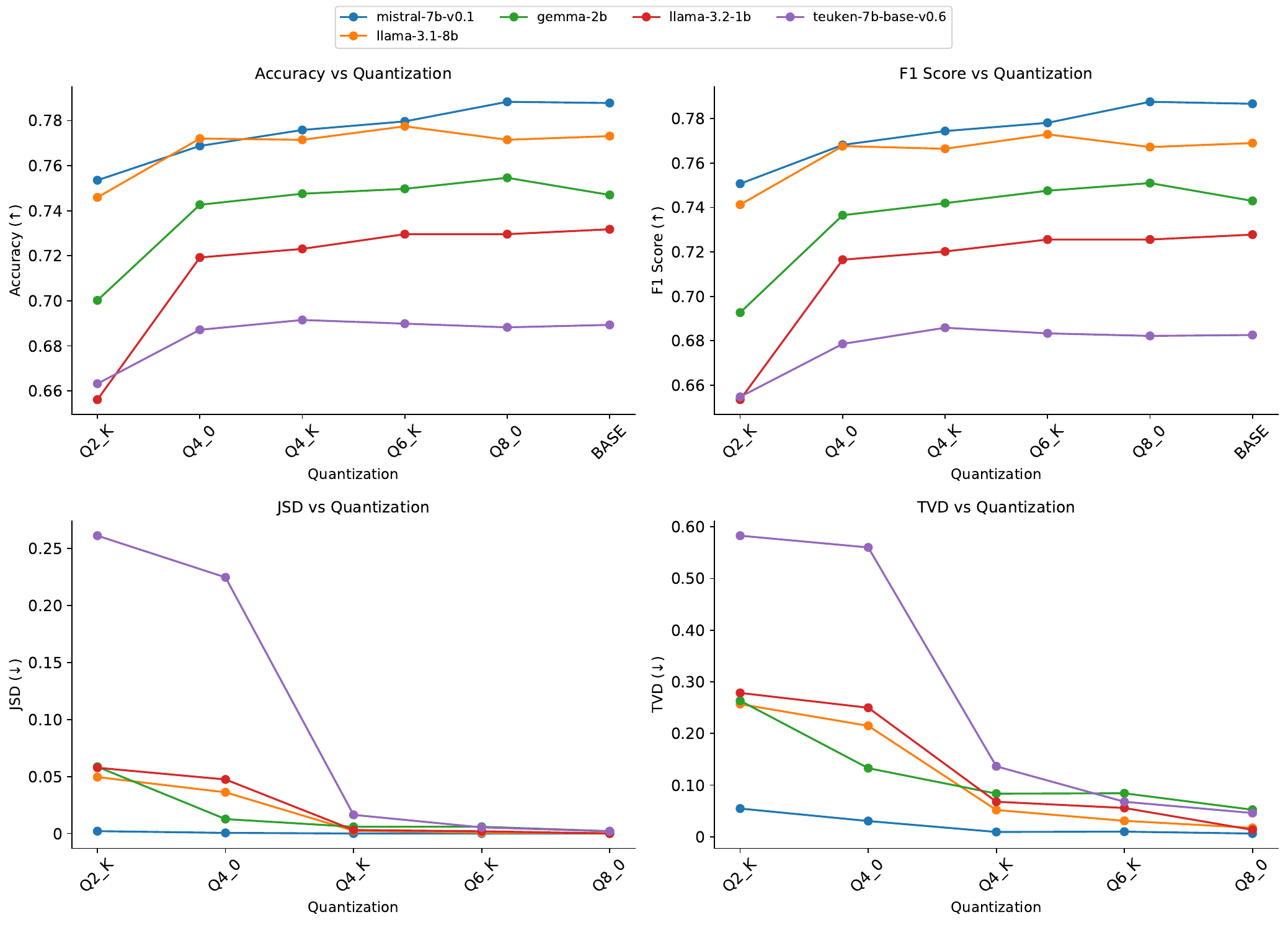} 
    \caption{Comparative analysis of task metrics (top) and divergence metrics (bottom) across quantization levels on PIQA. Accuracy and F1 can exhibit non-monotonic fluctuations, whereas TVD and JSD provide distribution-level fidelity measurements relative to the \BF reference that generally increase under stronger quantization in these experiments.}
    \label{fig:4panel_divergence}
\end{figure*}

\subsection{Quantifying Information Loss via TVD}
To complement task accuracy, we analyze the full-vocabulary softmax distributions of the models. We calculate TVD and JSD to quantify next-token distributional divergence from the \BF reference induced by quantization.
As illustrated in Figure \ref{fig:4panel_divergence}, the contrast between task-level and distribution-level metrics yields one of our most interesting results. The upper panels (Accuracy and F1 Score) exhibit erratic, non-monotonic behavior across quantization levels. 
For example, the accuracy of \textit{LLaMA-3.1-8B} increases slightly at the \qPTQ{6}{K} threshold relative to its uncompressed baseline before decreasing at lower precisions. This non-monotonic behavior illustrates why accuracy alone may be insufficient to characterize reference-distribution fidelity.

The lower panels of Figure \ref{fig:4panel_divergence} show that TVD and JSD generally increase as quantization becomes more aggressive in the evaluated PIQA settings. Because TVD measures distance from the \BF next-token distribution rather than agreement with class labels, it is not directly affected by class imbalance in the same way as accuracy. It nevertheless measures reference fidelity rather than task correctness or calibration.

TVD is particularly interpretable because it equals the fraction of probability mass that must be reassigned to transform one predictive distribution into the other. For example, when compressing \textit{LLaMA-3.1-8B} to \qPTQ{2}{K} on PIQA, accuracy remains $74.6\%$, compared with $77.3\%$ for \BF, while TVD is $0.257$. Thus, despite a moderate accuracy change, transforming the quantized next-token distribution into the \BF reference requires reassigning $25.7\%$ of probability mass.
\subsection{Mixed-Precision vs. Legacy Uniform Quantization}
We use TVD to compare the reference-distribution fidelity of modern mixed-precision K-quants (\qPTQ{4}{K}) and legacy uniform quantization (\qPTQ{4}{0}) in the evaluated settings.

When evaluating these two algorithms purely through the lens of deployment resources and task metrics, they appear practically identical. As detailed in Table \ref{tab:main_results_comprehensive}, compressing the \textit{LLaMA-3.1-8B} architecture via \qPTQ{4}{0} and \qPTQ{4}{K} yields nearly indistinguishable VRAM footprints (approx. $5.5$ GB) and essentially identical zero-shot accuracy scores on both the PIQA benchmark ($77.2\%$ vs. $77.1\%$) and the MMLU benchmark ($32.1\%$ vs. $32.2\%$). 
 
\begin{figure}[!htb]
    \centering
    \begin{subfigure}{0.48\linewidth}
        \centering
        \includegraphics[width=\linewidth]{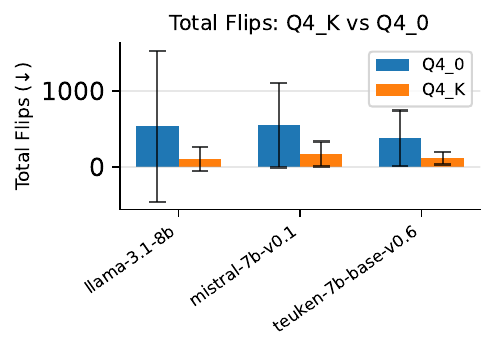}
    \end{subfigure} 
    \begin{subfigure}{0.48\linewidth}
        \centering
        \includegraphics[width=\linewidth]{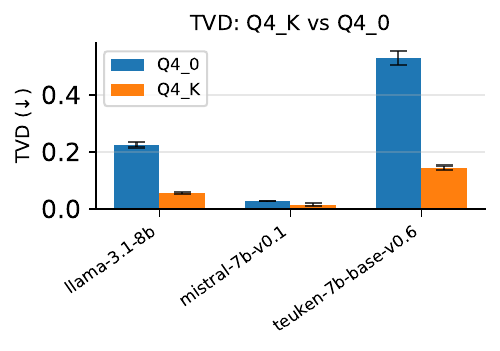}
    \end{subfigure}
     \caption{Comparison of legacy uniform quantization (\qPTQ{4}{0}) versus mixed-precision (\qPTQ{4}{K}) averaged across the evaluation suite in \textbf{a)} flips and \textbf{b)} TVD. Error bars denote the standard deviation across datasets. 
     }
    \label{fig:q4_showdown}
\end{figure}

However, Figure \ref{fig:q4_showdown} shows substantial differences in distributional fidelity and flip behavior. In the evaluated settings, \qPTQ{4}{K} exhibits lower average divergence, whereas \qPTQ{4}{0} shows greater variation in prediction flips across datasets.

For \textit{LLaMA-3.1-8B} on MMLU (Table \ref{tab:main_results_comprehensive}), \qPTQ{4}{K} yields a TVD of $0.055$, whereas \qPTQ{4}{0} yields $0.221$, approximately four times higher, despite similar task accuracy. For \textit{Teuken-7B}, the corresponding TVDs are $0.143$ for \qPTQ{4}{K} and $0.541$ for \qPTQ{4}{0}. These results indicate substantially lower fidelity to the \BF reference for uniform quantization in these evaluated cases.

These results show that TVD distinguishes the two quantization schemes even when their memory footprints and task accuracies are similar. For the tested models, benchmarks, and llama.cpp schemes, mixed-precision quantization generally provides greater fidelity to the \BF reference than uniform quantization.
\subsection{The Hardware Pareto Frontier}
Figure \ref{fig:pareto_frontier} relates VRAM requirements to TVD relative to the \BF reference in the evaluated settings. The plot shows that lower-memory quantization states are generally associated with higher divergence, although the magnitude varies across architectures and quantization schemes.

For example, \textit{LLaMA-3.1-8B} is reduced from 15.2 GB in \BF to 5.27 GB with \qPTQ{4}{K}, while its reported TVD remains below $0.055$ in the corresponding evaluation. Moving from \qPTQ{4}{K} to \qPTQ{4}{0} yields comparatively small memory reductions but substantially higher TVD for several models. More aggressive states, such as \qPTQ{2}{K}, show larger divergence in some architectures.

Accordingly, \qPTQ{4}{K} is a favorable empirical memory--fidelity trade-off among the tested models, benchmarks, and llama.cpp quantization schemes. This observation does not constitute a universal deployment threshold and does not substitute for latency, throughput, energy, calibration, or generation-quality evaluation on target devices.
\section{Conclusion and Future Work}

For quantized LLMs, zero-shot accuracy alone may not characterize fidelity to an uncompressed reference model. We therefore evaluate full-vocabulary predictive distributions using bounded JSD and TVD. JSD provides a bounded information-theoretic comparison, while TVD quantifies the probability mass that must be reassigned to transform the quantized distribution into the \BF reference distribution. These metrics measure reference fidelity rather than correctness, calibration, safety, or user-perceived quality.

Across the tested models, benchmarks, and llama.cpp quantization schemes, uniform \qPTQ{4}{0} often yields substantially higher TVD than mixed-precision \qPTQ{4}{K} despite similar task accuracy and memory footprints. Thus, \qPTQ{4}{K} provides a favorable empirical memory--fidelity trade-off in our setting, including up to a $65\%$ memory reduction for the reported LLaMA-3.1-8B configuration. This result is not a universal safe-deployment threshold.

\textbf{Future Work:} Our study is limited to base models, fixed prompt templates, prompt-conditioned full-vocabulary distributions, and llama.cpp/GGUF quantization. Future work should evaluate multi-step generation in instruction-tuned and RLHF-aligned architectures; compare AWQ, GPTQ, and SmoothQuant; assess prompt-template sensitivity; examine the computational scalability of full-vocabulary divergence evaluation; and measure latency, throughput, energy, calibration, and generation quality on target hardware.
\section*{Acknowledgment}
The authors acknowledge support from the PhenoRob Cluster of Excellence and the Lamarr Institute for Machine Learning and Artificial Intelligence. We thank G. Gerganov and the open-source community for developing and maintaining \texttt{llama.cpp} and \texttt{ggml}, which supported this work. ChatGPT \cite{openai2023gpt} and Gemini \cite{team2023gemini} were used for code optimization and language proofreading.

\bibliographystyle{unsrt}
\bibliography{references}

@article{team2024llama,
  title={The llama 3 herd of models},
  author={Grattafiori, Aaron and Dubey, Abhimanyu and Jauhri, Abhinav and Pandey, Abhinav and Kadian, Abhishek and Al-Dahle, Ahmad and Letman, Aiesha and Mathur, Akhil and Schelten, Alan and Vaughan, Alex and others},
  journal={arXiv preprint arXiv:2407.21783},
  year={2024}
}

@article{jiang2023mistral7b,
  author       = {Albert Q. Jiang and
                  Alexandre Sablayrolles and
                  Arthur Mensch and
                  Chris Bamford and
                  Devendra Singh Chaplot and
                  Diego de Las Casas and
                  Florian Bressand and
                  Gianna Lengyel and
                  Guillaume Lample and
                  Lucile Saulnier and
                  L{\'{e}}lio Renard Lavaud and
                  Marie{-}Anne Lachaux and
                  Pierre Stock and
                  Teven Le Scao and
                  Thibaut Lavril and
                  Thomas Wang and
                  Timoth{\'{e}}e Lacroix and
                  William El Sayed},
  title        = {Mistral 7B},
  journal      = {CoRR},
  volume       = {abs/2310.06825},
  year         = {2023},
  url          = {https://doi.org/10.48550/arXiv.2310.06825},
  doi          = {10.48550/ARXIV.2310.06825},
  eprinttype   = {arXiv},
  eprint       = {2310.06825}
}

@article{gemmateam2024gemma2improvingopen,
  title={Gemma 2: Improving open language models at a practical size},
  author={Gemma Team},
  url = {https://doi.org/10.48550/arXiv.2408.00118},
  journal={arXiv preprint arXiv:2408.00118},
  year={2024}
}

@inproceedings{
frantar2023gptq,
title={{OPTQ}: Accurate Quantization for Generative Pre-trained Transformers},
author={Elias Frantar and Saleh Ashkboos and Torsten Hoefler and Dan Alistarh},
booktitle={The Eleventh International Conference on Learning Representations },
year={2023},
url={https://openreview.net/forum?id=tcbBPnfwxS}
}

@article{lin2024awq,
  title={AWQ: Activation-aware weight quantization for on-device LLM compression and acceleration},
  author={Lin, Ji and Tang, Jiaming and Tang, Haotian and Yang, Shang and Xiao, Guangxuan and Han, Song},
  journal={GetMobile: Mobile Computing and Communications},
  volume={28},
  number={4},
  pages={12--17},
  year={2025},
  publisher={ACM New York, NY, USA}
}

@inproceedings{xiao2023smoothquant,
  author       = {Guangxuan Xiao and
                  Ji Lin and
                  Micka{\"{e}}l Seznec and
                  Hao Wu and
                  Julien Demouth and
                  Song Han},
  title        = {SmoothQuant: Accurate and Efficient Post-Training Quantization for
                  Large Language Models},
  booktitle    = {International Conference on Machine Learning, {ICML} 2023, 23-29 July
                  2023, Honolulu, Hawaii, {USA}},
  series       = {Proceedings of Machine Learning Research},
  volume       = {202},
  pages        = {38087--38099},
  publisher    = {{PMLR}},
  year         = {2023},
  url          = {https://proceedings.mlr.press/v202/xiao23c.html}
}

@inproceedings{chee2023quip,
  author       = {Jerry Chee and
                  Yaohui Cai and
                  Volodymyr Kuleshov and
                  Christopher De Sa},
  title        = {QuIP: 2-Bit Quantization of Large Language Models With Guarantees},
  booktitle    = {Advances in Neural Information Processing Systems 36: Annual Conference
                  on Neural Information Processing Systems 2023, NeurIPS 2023, New Orleans,
                  LA, USA, December 10 - 16, 2023},
  year         = {2023},
  url          = {http://papers.nips.cc/paper\_files/paper/2023/hash/0df38cd13520747e1e64e5b123a78ef8-Abstract-Conference.html}
}

@inproceedings{sparrenberg2025small,
  author       = {Lorenz Sparrenberg and
                  Tobias Deu{\ss}er and
                  Armin Berger and
                  Rafet Sifa},
  title        = {Small and Fast LLMs on Commodity Hardware: Post-Training Quantization
                  in llama. cpp},
  booktitle    = {12th {IEEE} International Conference on Data Science and Advanced
                  Analytics, {DSAA} 2025, Birmingham, United Kingdom, October 9-12,
                  2025},
  pages        = {1--10},
  publisher    = {{IEEE}},
  year         = {2025},
  url          = {https://doi.org/10.1109/DSAA65442.2025.11247985},
  doi          = {10.1109/DSAA65442.2025.11247985}
}

@article{gong2024llm,
  author       = {Ruihao Gong and
                  Yang Yong and
                  Shiqiao Gu and
                  Yushi Huang and
                  Yunchen Zhang and
                  Xianglong Liu and
                  Dacheng Tao},
  title        = {LLM-QBench: {A} Benchmark Towards the Best Practice for Post-training
                  Quantization of Large Language Models},
  journal      = {CoRR},
  volume       = {abs/2405.06001},
  year         = {2024},
  url          = {https://doi.org/10.48550/arXiv.2405.06001},
  doi          = {10.48550/ARXIV.2405.06001},
  eprinttype   = {arXiv},
  eprint       = {2405.06001}
}

@inproceedings{hendrycks2020measuring,
  author       = {Dan Hendrycks and
                  Collin Burns and
                  Steven Basart and
                  Andy Zou and
                  Mantas Mazeika and
                  Dawn Song and
                  Jacob Steinhardt},
  title        = {Measuring Massive Multitask Language Understanding},
  booktitle    = {9th International Conference on Learning Representations, {ICLR} 2021,
                  Virtual Event, Austria, May 3-7, 2021},
  publisher    = {OpenReview.net},
  year         = {2021},
  url          = {https://openreview.net/forum?id=d7KBjmI3GmQ}
}

@inproceedings{clark2019boolq,
  author       = {Christopher Clark and
                  Kenton Lee and
                  Ming{-}Wei Chang and
                  Tom Kwiatkowski and
                  Michael Collins and
                  Kristina Toutanova},
  title        = {BoolQ: Exploring the Surprising Difficulty of Natural Yes/No Questions},
  booktitle    = {Proceedings of the 2019 Conference of the North American Chapter of
                  the Association for Computational Linguistics: Human Language Technologies,
                  {NAACL-HLT} 2019, Minneapolis, MN, USA, June 2-7, 2019, Volume 1 (Long
                  and Short Papers)},
  pages        = {2924--2936},
  publisher    = {Association for Computational Linguistics},
  year         = {2019},
  url          = {https://doi.org/10.18653/v1/n19-1300},
  doi          = {10.18653/V1/N19-1300}
}

@inproceedings{bisk2020piqa,
  author       = {Yonatan Bisk and
                  Rowan Zellers and
                  Ronan Le Bras and
                  Jianfeng Gao and
                  Yejin Choi},
  title        = {{PIQA:} Reasoning about Physical Commonsense in Natural Language},
  booktitle    = {The Thirty-Fourth {AAAI} Conference on Artificial Intelligence, {AAAI}
                  2020, The Thirty-Second Innovative Applications of Artificial Intelligence
                  Conference, {IAAI} 2020, The Tenth {AAAI} Symposium on Educational
                  Advances in Artificial Intelligence, {EAAI} 2020, New York, NY, USA,
                  February 7-12, 2020},
  pages        = {7432--7439},
  publisher    = {{AAAI} Press},
  year         = {2020},
  url          = {https://doi.org/10.1609/aaai.v34i05.6239},
  doi          = {10.1609/AAAI.V34I05.6239}
}

@article{liang2023holistic,
  author       = {Percy Liang and
                  Rishi Bommasani and
                  Tony Lee and
                  Dimitris Tsipras and
                  Dilara Soylu and
                  Michihiro Yasunaga and
                  Yian Zhang and
                  Deepak Narayanan and
                  Yuhuai Wu and
                  Ananya Kumar and
                  Benjamin Newman and
                  Binhang Yuan and
                  Bobby Yan and
                  Ce Zhang and
                  Christian Cosgrove and
                  Christopher D. Manning and
                  Christopher R{\'{e}} and
                  Diana Acosta{-}Navas and
                  Drew A. Hudson and
                  Eric Zelikman and
                  Esin Durmus and
                  Faisal Ladhak and
                  Frieda Rong and
                  Hongyu Ren and
                  Huaxiu Yao and
                  Jue Wang and
                  Keshav Santhanam and
                  Laurel J. Orr and
                  Lucia Zheng and
                  Mert Y{\"{u}}ksekg{\"{o}}n{\"{u}}l and
                  Mirac Suzgun and
                  Nathan Kim and
                  Neel Guha and
                  Niladri S. Chatterji and
                  Omar Khattab and
                  Peter Henderson and
                  Qian Huang and
                  Ryan Chi and
                  Sang Michael Xie and
                  Shibani Santurkar and
                  Surya Ganguli and
                  Tatsunori Hashimoto and
                  Thomas Icard and
                  Tianyi Zhang and
                  Vishrav Chaudhary and
                  William Wang and
                  Xuechen Li and
                  Yifan Mai and
                  Yuhui Zhang and
                  Yuta Koreeda},
  title        = {Holistic Evaluation of Language Models},
  journal      = {Trans. Mach. Learn. Res.},
  volume       = {2023},
  year         = {2023},
  url          = {https://openreview.net/forum?id=iO4LZibEqW}
}

@article{huang2024good,
  author       = {Wei Huang and
                  Xudong Ma and
                  Haotong Qin and
                  Xingyu Zheng and
                  Chengtao Lv and
                  Hong Chen and
                  Jie Luo and
                  Xiaojuan Qi and
                  Xianglong Liu and
                  Michele Magno},
  title        = {How Good Are Low-bit Quantized LLaMA3 Models? An Empirical Study},
  journal      = {CoRR},
  volume       = {abs/2404.14047},
  year         = {2024},
  url          = {https://doi.org/10.48550/arXiv.2404.14047},
  doi          = {10.48550/ARXIV.2404.14047},
  eprinttype   = {arXiv},
  eprint       = {2404.14047}
}

@inproceedings{dutta2024accuracy,
  author       = {Abhinav Dutta and
                  Sanjeev Krishnan and
                  Nipun Kwatra and
                  Ramachandran Ramjee},
  title        = {Accuracy is Not All You Need},
  booktitle    = {Advances in Neural Information Processing Systems 37: Annual Conference
                  on Neural Information Processing Systems 2024, NeurIPS 2024, Vancouver,
                  BC, Canada, December 10 - 15, 2024},
  year         = {2024},
  url          = {http://papers.nips.cc/paper\_files/paper/2024/hash/e0e956681b04ac126679e8c7dd706b2e-Abstract-Conference.html}
}

@article{kubler2026llms,
  author       = {Jonas M. K{\"{u}}bler and
                  Kailash Budhathoki and
                  Matth{\"{a}}us Kleindessner and
                  Xiong Zhou and
                  Junming Yin and
                  Ashish Khetan and
                  George Karypis},
  title        = {When LLMs get significantly worse: {A} statistical approach to detect
                  model degradations},
  journal      = {CoRR},
  volume       = {abs/2602.10144},
  year         = {2026},
  url          = {https://doi.org/10.48550/arXiv.2602.10144},
  doi          = {10.48550/ARXIV.2602.10144},
  eprinttype   = {arXiv},
  eprint       = {2602.10144}
}

@article{hooker2019compressed,
  title={What do compressed deep neural networks forget?},
  author={Hooker, Sara and Courville, Aaron and Clark, Gregory and Dauphin, Yann and Frome, Andrea},
  journal={arXiv preprint arXiv:1911.05248},
  year={2019}
}

@inproceedings{guo2017calibration,
  author       = {Chuan Guo and
                  Geoff Pleiss and
                  Yu Sun and
                  Kilian Q. Weinberger},
  title        = {On Calibration of Modern Neural Networks},
  booktitle    = {Proceedings of the 34th International Conference on Machine Learning,
                  {ICML} 2017, Sydney, NSW, Australia, 6-11 August 2017},
  series       = {Proceedings of Machine Learning Research},
  volume       = {70},
  pages        = {1321--1330},
  publisher    = {{PMLR}},
  year         = {2017},
  url          = {http://proceedings.mlr.press/v70/guo17a.html}
}

@inproceedings{szegedy2014intriguing,
  author       = {Christian Szegedy and
                  Wojciech Zaremba and
                  Ilya Sutskever and
                  Joan Bruna and
                  Dumitru Erhan and
                  Ian J. Goodfellow and
                  Rob Fergus},
  title        = {Intriguing properties of neural networks},
  booktitle    = {2nd International Conference on Learning Representations, {ICLR} 2014,
                  Banff, AB, Canada, April 14-16, 2014, Conference Track Proceedings},
  year         = {2014},
  url          = {http://arxiv.org/abs/1312.6199}
}

@article{begoli2019need,
  author       = {Edmon Begoli and
                  Tanmoy Bhattacharya and
                  Dimitri Kusnezov},
  title        = {The need for uncertainty quantification in machine-assisted medical
                  decision making},
  journal      = {Nat. Mach. Intell.},
  volume       = {1},
  number       = {1},
  pages        = {20--23},
  year         = {2019},
  url          = {https://doi.org/10.1038/s42256-018-0004-1},
  doi          = {10.1038/S42256-018-0004-1}
}

@article{kadavath2022language,
  author       = {Saurav Kadavath and
                  Tom Conerly and
                  Amanda Askell and
                  Tom Henighan and
                  Dawn Drain and
                  Ethan Perez and
                  Nicholas Schiefer and
                  Zac Hatfield{-}Dodds and
                  Nova DasSarma and
                  Eli Tran{-}Johnson and
                  Scott Johnston and
                  Sheer El Showk and
                  Andy Jones and
                  Nelson Elhage and
                  Tristan Hume and
                  Anna Chen and
                  Yuntao Bai and
                  Sam Bowman and
                  Stanislav Fort and
                  Deep Ganguli and
                  Danny Hernandez and
                  Josh Jacobson and
                  Jackson Kernion and
                  Shauna Kravec and
                  Liane Lovitt and
                  Kamal Ndousse and
                  Catherine Olsson and
                  Sam Ringer and
                  Dario Amodei and
                  Tom Brown and
                  Jack Clark and
                  Nicholas Joseph and
                  Ben Mann and
                  Sam McCandlish and
                  Chris Olah and
                  Jared Kaplan},
  title        = {Language Models (Mostly) Know What They Know},
  journal      = {CoRR},
  volume       = {abs/2207.05221},
  year         = {2022},
  url          = {https://doi.org/10.48550/arXiv.2207.05221},
  doi          = {10.48550/ARXIV.2207.05221},
  eprinttype   = {arXiv},
  eprint       = {2207.05221}
}

@article{lin1991divergence,
  author       = {Jianhua Lin},
  title        = {Divergence measures based on the Shannon entropy},
  journal      = {{IEEE} Trans. Inf. Theory},
  volume       = {37},
  number       = {1},
  pages        = {145--151},
  year         = {1991},
  url          = {https://doi.org/10.1109/18.61115},
  doi          = {10.1109/18.61115}
}

@book{cover2006elements,
  author       = {Thomas M. Cover and
                  Joy A. Thomas},
  title        = {Elements of information theory {(2.} ed.)},
  publisher    = {Wiley},
  year         = {2006},
  url          = {http://www.elementsofinformationtheory.com/},
  isbn         = {978-0-471-24195-9}
}

@inproceedings{brown2020language,
  author       = {Tom B. Brown and
                  Benjamin Mann and
                  Nick Ryder and
                  Melanie Subbiah and
                  Jared Kaplan and
                  Prafulla Dhariwal and
                  Arvind Neelakantan and
                  Pranav Shyam and
                  Girish Sastry and
                  Amanda Askell and
                  Sandhini Agarwal and
                  Ariel Herbert{-}Voss and
                  Gretchen Krueger and
                  Tom Henighan and
                  Rewon Child and
                  Aditya Ramesh and
                  Daniel M. Ziegler and
                  Jeffrey Wu and
                  Clemens Winter and
                  Christopher Hesse and
                  Mark Chen and
                  Eric Sigler and
                  Mateusz Litwin and
                  Scott Gray and
                  Benjamin Chess and
                  Jack Clark and
                  Christopher Berner and
                  Sam McCandlish and
                  Alec Radford and
                  Ilya Sutskever and
                  Dario Amodei},
  title        = {Language Models are Few-Shot Learners},
  booktitle    = {Advances in Neural Information Processing Systems 33: Annual Conference
                  on Neural Information Processing Systems 2020, NeurIPS 2020, December
                  6-12, 2020, virtual},
  year         = {2020},
  url          = {https://proceedings.neurips.cc/paper/2020/hash/1457c0d6bfcb4967418bfb8ac142f64a-Abstract.html}
}

@article{llama2023llama,
  author       = {Hugo Touvron and
                  Thibaut Lavril and
                  Gautier Izacard and
                  Xavier Martinet and
                  Marie{-}Anne Lachaux and
                  Timoth{\'{e}}e Lacroix and
                  Baptiste Rozi{\`{e}}re and
                  Naman Goyal and
                  Eric Hambro and
                  Faisal Azhar and
                  Aur{\'{e}}lien Rodriguez and
                  Armand Joulin and
                  Edouard Grave and
                  Guillaume Lample},
  title        = {LLaMA: Open and Efficient Foundation Language Models},
  journal      = {CoRR},
  volume       = {abs/2302.13971},
  year         = {2023},
  url          = {https://doi.org/10.48550/arXiv.2302.13971},
  doi          = {10.48550/ARXIV.2302.13971},
  eprinttype   = {arXiv},
  eprint       = {2302.13971}
}

@inproceedings{carlini2021extracting,
  author       = {Nicholas Carlini and
                  Florian Tram{\`{e}}r and
                  Eric Wallace and
                  Matthew Jagielski and
                  Ariel Herbert{-}Voss and
                  Katherine Lee and
                  Adam Roberts and
                  Tom B. Brown and
                  Dawn Song and
                  {\'{U}}lfar Erlingsson and
                  Alina Oprea and
                  Colin Raffel},
  title        = {Extracting Training Data from Large Language Models},
  booktitle    = {30th {USENIX} Security Symposium, {USENIX} Security 2021, August 11-13,
                  2021},
  pages        = {2633--2650},
  publisher    = {{USENIX} Association},
  year         = {2021},
  url          = {https://www.usenix.org/conference/usenixsecurity21/presentation/carlini-extracting}
}

@article{mireshghallah2020privacy,
  author       = {Fatemehsadat Mireshghallah and
                  Mohammadkazem Taram and
                  Praneeth Vepakomma and
                  Abhishek Singh and
                  Ramesh Raskar and
                  Hadi Esmaeilzadeh},
  title        = {Privacy in Deep Learning: {A} Survey},
  journal      = {CoRR},
  volume       = {abs/2004.12254},
  year         = {2020},
  url          = {https://arxiv.org/abs/2004.12254},
  eprinttype   = {arXiv},
  eprint       = {2004.12254}
}

@article{chen2026device,
  author       = {Wanyi Chen and
                  Junhao Wang and
                  Yiwei Zhang and
                  Yufan Shi and
                  Tianyi Jiang and
                  Shengxian Zhou and
                  Chenxu Wu and
                  Andi Zhang and
                  Chenyue Zhou and
                  Minxuan Wang and
                  Xinyu Liu and
                  Xiaoshuai Hao and
                  Yinan Wu and
                  Yichen Li and
                  Yuwei Hu and
                  Zhao Cao and
                  Yang Lu and
                  Mengke Li and
                  Yanbiao Ma and
                  Zhiwu Lu and
                  Jungong Han and
                  Yike Guo},
  title        = {On-device large language models: a survey of model compression and
                  system optimization},
  journal      = {Artif. Intell. Rev.},
  volume       = {59},
  number       = {9},
  pages        = {191},
  year         = {2026},
  url          = {https://doi.org/10.1007/s10462-026-11538-1},
  doi          = {10.1007/S10462-026-11538-1}
}

@article{zhu2024survey,
  author       = {Xunyu Zhu and
                  Jian Li and
                  Yong Liu and
                  Can Ma and
                  Weiping Wang},
  title        = {A Survey on Model Compression for Large Language Models},
  journal      = {Trans. Assoc. Comput. Linguistics},
  volume       = {12},
  pages        = {1556--1577},
  year         = {2024},
  url          = {https://doi.org/10.1162/tacl\_a\_00704},
  doi          = {10.1162/TACL\_A\_00704}
}

@inproceedings{nagel2020up,
  author       = {Markus Nagel and
                  Rana Ali Amjad and
                  Mart van Baalen and
                  Christos Louizos and
                  Tijmen Blankevoort},
  title        = {Up or Down? Adaptive Rounding for Post-Training Quantization},
  booktitle    = {Proceedings of the 37th International Conference on Machine Learning,
                  {ICML} 2020, 13-18 July 2020, Virtual Event},
  series       = {Proceedings of Machine Learning Research},
  volume       = {119},
  pages        = {7197--7206},
  publisher    = {{PMLR}},
  year         = {2020},
  url          = {http://proceedings.mlr.press/v119/nagel20a.html}
}

@article{schaeffer2023emergent,
  title={Are emergent abilities of large language models a mirage?},
  author={Schaeffer, Rylan and Miranda, Brando and Koyejo, Sanmi},
  journal={Advances in neural information processing systems},
  volume={36},
  pages={55565--55581},
  year={2023}
}

@inproceedings{ali2025teuken,
  author       = {Mehdi Ali and
                  Michael Fromm and
                  Klaudia Thellmann and
                  Jan Ebert and
                  Alexander Arno Weber and
                  Richard Rutmann and
                  Charvi Jain and
                  Max L{\"{u}}bbering and
                  Daniel Steinigen and
                  Johannes Leveling and
                  Katrin Klug and
                  Jasper Schulze Buschhoff and
                  Lena Jurkschat and
                  Hammam Abdelwahab and
                  Benny J{\"{o}}rg Stein and
                  Karl{-}Heinz Sylla and
                  Pavel Denisov and
                  Nicolo' Brandizzi and
                  Qasid Saleem and
                  Anirban Bhowmick and
                  Lennard Helmer and
                  Chelsea Maria John and
                  Pedro Ortiz Suarez and
                  Malte Ostendorff and
                  Alex Jude and
                  Lalith Manjunath and
                  Samuel Weinbach and
                  Carolin Penke and
                  Oleg Filatov and
                  Fabio Barth and
                  Paramita Mirza and
                  Lucas Weber and
                  Ines Wendler and
                  Rafet Sifa and
                  Fabian K{\"{u}}ch and
                  Andreas Herten and
                  Ren{\'{e}} J{\"{a}}kel and
                  Georg Rehm and
                  Stefan Kesselheim and
                  Joachim K{\"{o}}hler and
                  Nicolas Flores{-}Herr},
  title        = {Teuken-7B-Base {\&} Teuken-7B-Instruct: Towards European LLMs},
  booktitle    = {{ECAI} 2025 - 28th European Conference on Artificial Intelligence,
                  25-30 October 2025, Bologna, Italy - Including 14th Conference on
                  Prestigious Applications of Intelligent Systems {(PAIS} 2025)},
  series       = {Frontiers in Artificial Intelligence and Applications},
  volume       = {413},
  pages        = {4321--4329},
  publisher    = {{IOS} Press},
  year         = {2025},
  url          = {https://doi.org/10.3233/FAIA251328},
  doi          = {10.3233/FAIA251328}
}

@inproceedings{wang2024causalbench,
    title = "{C}ausal{B}ench: A Comprehensive Benchmark for Evaluating Causal Reasoning Capabilities of Large Language Models",
    author = "Wang, Zeyu",
    booktitle = "Proceedings of the 10th SIGHAN Workshop on Chinese Language Processing (SIGHAN-10)",
    month = aug,
    year = "2024",
    address = "Bangkok, Thailand",
    publisher = "Association for Computational Linguistics",
    url = "https://aclanthology.org/2024.sighan-1.17/",
    pages = "143--151",
}

@article{openai2023gpt,
  author       = {OpenAI},
  title        = {{GPT-4} Technical Report},
  journal      = {CoRR},
  volume       = {abs/2303.08774},
  year         = {2023},
  url          = {https://doi.org/10.48550/arXiv.2303.08774},
  doi          = {10.48550/ARXIV.2303.08774},
  eprinttype   = {arXiv},
  eprint       = {2303.08774}
}

@article{team2023gemini,
  author       = {Gemini Team},
  title        = {Gemini: {A} Family of Highly Capable Multimodal Models},
  journal      = {CoRR},
  volume       = {abs/2312.11805},
  year         = {2023},
  url          = {https://doi.org/10.48550/arXiv.2312.11805},
  doi          = {10.48550/ARXIV.2312.11805},
  eprinttype   = {arXiv},
  eprint       = {2312.11805}
}

@misc{gerganov2023llama,
  author       = {Georgi Gerganov and others},
  title        = {llama.cpp: Port of Facebook's LLaMA model in C/C++},
  year         = {2023},
  howpublished = {\url{https://github.com/ggerganov/llama.cpp}},
  note         = {Accessed: Mar. 26, 2026, Commit: ded446b}
}

@inproceedings{chopra2025small,
  author       = {Muskaan Chopra and
                  Lorenz Sparrenberg and
                  Sarthak Khanna and
                  Rafet Sifa},
  title        = {How Small Can You Go? Compact Language Models for On-Device Critical
                  Error Detection in Machine Translation},
  booktitle    = {{IEEE} International Conference on Big Data, BigData 2025, Macau,
                  China, December 8-11, 2025},
  pages        = {5410--5417},
  publisher    = {{IEEE}},
  year         = {2025},
  url          = {https://doi.org/10.1109/BigData66926.2025.11401605},
  doi          = {10.1109/BIGDATA66926.2025.11401605}
}

@inproceedings{sparrenberg2025towards,
  author       = {Lorenz Sparrenberg and
                  Tobias Schneider and
                  Tobias Deu{\ss}er and
                  Armin Berger and
                  Rafet Sifa},
  title        = {Towards Uncertainty-Aware Low-Bit Quantized LLMs for On-Device Inference},
  booktitle    = {{IEEE} International Conference on Big Data, BigData 2025, Macau,
                  China, December 8-11, 2025},
  pages        = {5930--5939},
  publisher    = {{IEEE}},
  year         = {2025},
  url          = {https://doi.org/10.1109/BigData66926.2025.11400739},
  doi          = {10.1109/BIGDATA66926.2025.11400739}
}

@article{rababah2026illusion,
  author       = {Baha Rababah and Shahzeb Qamar and
                  Lorenz Sparrenberg and
                  Rafet Sifa and
                  Murat Kantarcioglu and
                  Cuneyt Gurcan Akcora and
                  Carson K. Leung},
  title        = {The Illusion of Equivalency: Statistical Characterization of Quantization
                  Effects in LLMs},
  journal      = {CoRR},
  volume       = {abs/2607.08734},
  year         = {2026},
  url          = {https://doi.org/10.48550/arXiv.2607.08734},
  doi          = {10.48550/ARXIV.2607.08734},
  eprinttype   = {arXiv},
  eprint       = {2607.08734},
  bibsource    = {dblp computer science bibliography, https://dblp.org}
}

@article{xu2026alignment,
  author       = {Bruce Changlong Xu and
                  Adarsh Kumarappan and
                  Mu Zhou},
  title        = {Alignment Collapse Under {KV} Cache Quantization: Diagnosis and Mitigation},
  journal      = {CoRR},
  volume       = {abs/2606.09864},
  year         = {2026},
  url          = {https://doi.org/10.48550/arXiv.2606.09864},
  doi          = {10.48550/ARXIV.2606.09864},
  eprinttype   = {arXiv},
  eprint       = {2606.09864},
  bibsource    = {dblp computer science bibliography, https://dblp.org}
}
 
\section*{Supplementary Material}

\section{Theorems}
\label{sec:theorems}
\begin{proof}
    For JSD, let $ M=(P+Q)/2, $ so $ KM=(KP+KQ)/2. $ We have \begin{align} JSD(KP,KQ) = \frac12 KL(KP\|KM) + \frac12 KL(KQ\|KM). \end{align}  We show $ KL(KP\|KM)\le KL(P\|M). $ 
    
    For each $a$, apply the log-sum inequality to $ x_v=K(a\mid v)P(v) $ and $ y_v=K(a\mid v)M(v): $ 
    
    \begin{align} \left( \sum_v x_v \right) \log \frac{\sum_v x_v}{\sum_v y_v} \le \sum_v x_v \log \frac{x_v}{y_v}. 
    \end{align}  
    
    Since $ \sum_v x_v=(KP)(a) $ and $ \sum_v y_v=(KM)(a), $ this yields 
    \begin{align} (KP)(a) \log \frac{(KP)(a)}{(KM)(a)} \le \sum_v K(a\mid v)P(v) \log \frac{P(v)}{M(v)}. \end{align}  
    
    Summing over $a$ and using $ \sum_a K(a\mid v)=1 $ gives $ KL(KP\|KM)\le KL(P\|M). $ The same argument with $Q$ in place of $P$ gives $ KL(KQ\|KM)\le KL(Q\|M). $ Averaging the two inequalities proves $ JSD(KP,KQ)\le JSD(P,Q). $
\end{proof}

\section{Complete Empirical Results}
\label{app:complete_results}

This appendix provides the exhaustive quantitative evaluation across all 120 experimental configurations discussed in this study. To ensure readability, the results are categorized by benchmark dataset. Tables \ref{tab:app_mmlu} through \ref{tab:app_causalbench} detail the zero-shot Accuracy, F1 Score, percentage of Prediction Flips, JSD, and TVD for all five foundation models under progressive quantization regimes.

\begin{table*}[htbp]
    \centering

    \begin{minipage}{0.48\linewidth}
        \centering
        \caption{Complete Evaluation Metrics: \textbf{MMLU}}
        \label{tab:app_mmlu}
        \renewcommand{\arraystretch}{1.05}
        \resizebox{\linewidth}{!}{
        \begin{tabular}{@{}l|c|ccccc@{}}
        \toprule
        \textbf{Model \& Quantization} & \textbf{VRAM} & \textbf{Acc ($\uparrow$)} & \textbf{F1 ($\uparrow$)} & \textbf{\% Flips ($\downarrow$)} & \textbf{JSD ($\downarrow$)} & \textbf{TVD ($\downarrow$)} \\ \midrule

        \multicolumn{7}{l}{\textbf{Teuken-7B-base-v0.6}} \\ \midrule
        Base (\BF) & $15.0$ GB & $31.1\%$ & $0.310$ & --- & --- & --- \\
        \qPTQ{8}{0} (Uniform) & $10.4$ GB & $30.8\%$ & $0.308$ & $\mathbf{1.05\%}$ & $\mathbf{0.002}$ & $\mathbf{0.048}$ \\
        \qPTQ{6}{K} (Mixed) & $9.2$ GB & $\mathbf{31.4\%}$ & $\mathbf{0.313}$ & $1.76\%$ & $0.006$ & $0.075$ \\
        \qPTQ{4}{K} (Mixed) & $7.7$ GB & $30.3\%$ & $0.302$ & $3.92\%$ & $0.018$ & $0.143$ \\
        \qPTQ{4}{0} (Uniform) & $7.2$ GB & $30.3\%$ & $0.303$ & $4.58\%$ & $0.213$ & $0.541$ \\
        \qPTQ{2}{K} (Mixed) & $6.3$ GB & $28.0\%$ & $0.279$ & $10.59\%$ & $0.262$ & $0.585$ \\
        \midrule

        \multicolumn{7}{l}{\textbf{LLaMA-3.1-8B}} \\ \midrule
        Base (\BF) & $14.9$ GB & $32.1\%$ & $0.320$ & --- & --- & --- \\
        \qPTQ{8}{0} (Uniform) & $9.4$ GB & $\mathbf{32.3\%}$ & $\mathbf{0.323}$ & $\mathbf{0.65\%}$ & $\mathbf{0.000}$ & $\mathbf{0.021}$ \\
        \qPTQ{6}{K} (Mixed) & $7.7$ GB & $32.2\%$ & $0.322$ & $0.78\%$ & $0.001$ & $0.034$ \\
        \qPTQ{4}{K} (Mixed) & $6.3$ GB & $32.2\%$ & $0.322$ & $2.48\%$ & $0.003$ & $0.055$ \\
        \qPTQ{4}{0} (Uniform) & $6.0$ GB & $32.1\%$ & $0.321$ & $2.94\%$ & $0.038$ & $0.221$ \\
        \qPTQ{2}{K} (Mixed) & $4.8$ GB & $31.2\%$ & $0.312$ & $9.87\%$ & $0.050$ & $0.258$ \\
        \midrule

        \multicolumn{7}{l}{\textbf{Mistral-7B-v0.1}} \\ \midrule
        Base (\BF) & $14.0$ GB & $\mathbf{33.5\%}$ & $\mathbf{0.335}$ & --- & --- & --- \\
        \qPTQ{8}{0} (Uniform) & $8.1$ GB & $33.4\%$ & $0.334$ & $\mathbf{0.33\%}$ & $\mathbf{0.000}$ & $\mathbf{0.009}$ \\
        \qPTQ{6}{K} (Mixed) & $6.5$ GB & $33.0\%$ & $0.329$ & $1.83\%$ & $0.000$ & $0.016$ \\
        \qPTQ{4}{K} (Mixed) & $5.0$ GB & $32.8\%$ & $0.327$ & $3.46\%$ & $0.000$ & $0.015$ \\
        \qPTQ{4}{0} (Uniform) & $4.8$ GB & $31.9\%$ & $0.319$ & $6.08\%$ & $0.001$ & $0.030$ \\
        \qPTQ{2}{K} (Mixed) & $3.5$ GB & $31.3\%$ & $0.313$ & $7.58\%$ & $0.002$ & $0.057$ \\
        \midrule

        \multicolumn{7}{l}{\textbf{Gemma-2B}} \\ \midrule
        Base (\BF) & $5.8$ GB & $\mathbf{29.2\%}$ & $\mathbf{0.292}$ & --- & --- & --- \\
        \qPTQ{8}{0} (Uniform) & $4.6$ GB & $\mathbf{29.2\%}$ & $0.292$ & $\mathbf{0.52\%}$ & $\mathbf{0.002}$ & $\mathbf{0.050}$ \\
        \qPTQ{6}{K} (Mixed) & $4.0$ GB & $29.0\%$ & $0.290$ & $2.16\%$ & $0.006$ & $0.089$ \\
        \qPTQ{4}{K} (Mixed) & $3.6$ GB & $29.1\%$ & $0.291$ & $3.27\%$ & $0.006$ & $0.086$ \\
        \qPTQ{4}{0} (Uniform) & $3.5$ GB & $28.7\%$ & $0.287$ & $5.16\%$ & $0.014$ & $0.143$ \\
        \qPTQ{2}{K} (Mixed) & $3.2$ GB & $27.5\%$ & $0.275$ & $11.50\%$ & $0.061$ & $0.266$ \\
        \midrule

        \multicolumn{7}{l}{\textbf{LLaMA-3.2-1B}} \\ \midrule
        Base (\BF) & $3.1$ GB & $28.8\%$ & $0.288$ & --- & --- & --- \\
        \qPTQ{8}{0} (Uniform) & $2.5$ GB & $\mathbf{28.9\%}$ & $\mathbf{0.289}$ & $\mathbf{0.72\%}$ & $\mathbf{0.000}$ & $\mathbf{0.013}$ \\
        \qPTQ{6}{K} (Mixed) & $2.2$ GB & $\mathbf{28.9\%}$ & $0.289$ & $1.24\%$ & $0.002$ & $0.052$ \\
        \qPTQ{4}{K} (Mixed) & $2.0$ GB & $28.7\%$ & $0.286$ & $3.33\%$ & $0.003$ & $0.066$ \\
        \qPTQ{4}{0} (Uniform) & $2.0$ GB & $28.3\%$ & $0.282$ & $4.90\%$ & $0.048$ & $0.251$ \\
        \qPTQ{2}{K} (Mixed) & $1.8$ GB & $26.9\%$ & $0.269$ & $13.07\%$ & $0.058$ & $0.279$ \\
        \bottomrule
        \end{tabular}}
    \end{minipage}
\hfill
    \begin{minipage}{0.48\linewidth}
        \centering
        \caption{Complete Evaluation Metrics: \textbf{PIQA}}
        \label{tab:app_piqa}
        \renewcommand{\arraystretch}{1.05}
        \resizebox{\linewidth}{!}{
        \begin{tabular}{@{}l|c|ccccc@{}}
        \toprule
        \textbf{Model \& Quantization} & \textbf{VRAM} & \textbf{Acc ($\uparrow$)} & \textbf{F1 ($\uparrow$)} & \textbf{\% Flips ($\downarrow$)} & \textbf{JSD ($\downarrow$)} & \textbf{TVD ($\downarrow$)} \\ \midrule

        \multicolumn{7}{l}{\textbf{Teuken-7B-base-v0.6}} \\ \midrule
        Base (\BF) & $14.8$ GB & $68.9\%$ & $0.683$ & --- & --- & --- \\
        \qPTQ{8}{0} (Uniform) & $8.3$ GB & $68.8\%$ & $0.682$ & $\mathbf{1.47\%}$ & $\mathbf{0.002}$ & $\mathbf{0.046}$ \\
        \qPTQ{6}{K} (Mixed) & $7.0$ GB & $69.0\%$ & $0.683$ & $1.85\%$ & $0.005$ & $0.068$ \\
        \qPTQ{4}{K} (Mixed) & $5.6$ GB & $\mathbf{69.2\%}$ & $\mathbf{0.686}$ & $3.43\%$ & $0.016$ & $0.136$ \\
        \qPTQ{4}{0} (Uniform) & $5.1$ GB & $68.7\%$ & $0.679$ & $4.84\%$ & $0.225$ & $0.560$ \\
        \qPTQ{2}{K} (Mixed) & $4.1$ GB & $66.3\%$ & $0.655$ & $10.83\%$ & $0.261$ & $0.583$ \\
        \midrule

        \multicolumn{7}{l}{\textbf{LLaMA-3.1-8B}} \\ \midrule
        Base (\BF) & $14.8$ GB & $77.3\%$ & $0.769$ & --- & --- & --- \\
        \qPTQ{8}{0} (Uniform) & $8.3$ GB & $77.1\%$ & $0.767$ & $\mathbf{1.09\%}$ & $\mathbf{0.000}$ & $\mathbf{0.017}$ \\
        \qPTQ{6}{K} (Mixed) & $6.6$ GB & $\mathbf{77.7\%}$ & $\mathbf{0.773}$ & $1.58\%$ & $0.001$ & $0.031$ \\
        \qPTQ{4}{K} (Mixed) & $5.2$ GB & $77.1\%$ & $0.766$ & $2.72\%$ & $0.002$ & $0.052$ \\
        \qPTQ{4}{0} (Uniform) & $4.9$ GB & $77.2\%$ & $0.768$ & $4.08\%$ & $0.036$ & $0.215$ \\
        \qPTQ{2}{K} (Mixed) & $3.6$ GB & $74.6\%$ & $0.741$ & $10.72\%$ & $0.050$ & $0.257$ \\
        \midrule

        \multicolumn{7}{l}{\textbf{Mistral-7B-v0.1}} \\ \midrule
        Base (\BF) & $14.0$ GB & $78.8\%$ & $0.787$ & --- & --- & --- \\
        \qPTQ{8}{0} (Uniform) & $7.8$ GB & $\mathbf{78.8\%}$ & $\mathbf{0.788}$ & $\mathbf{1.25\%}$ & $\mathbf{0.000}$ & $\mathbf{0.006}$ \\
        \qPTQ{6}{K} (Mixed) & $6.2$ GB & $78.0\%$ & $0.778$ & $2.01\%$ & $0.000$ & $0.010$ \\
        \qPTQ{4}{K} (Mixed) & $4.8$ GB & $77.6\%$ & $0.774$ & $3.26\%$ & $0.000$ & $0.009$ \\
        \qPTQ{4}{0} (Uniform) & $4.5$ GB & $76.9\%$ & $0.768$ & $4.08\%$ & $0.001$ & $0.030$ \\
        \qPTQ{2}{K} (Mixed) & $3.3$ GB & $75.4\%$ & $0.751$ & $6.91\%$ & $0.002$ & $0.055$ \\
        \midrule

        \multicolumn{7}{l}{\textbf{Gemma-2B}} \\ \midrule
        Base (\BF) & $5.5$ GB & $74.7\%$ & $0.743$ & --- & --- & --- \\
        \qPTQ{8}{0} (Uniform) & $3.4$ GB & $\mathbf{75.5\%}$ & $\mathbf{0.751}$ & $\mathbf{1.96\%}$ & $\mathbf{0.002}$ & $\mathbf{0.052}$ \\
        \qPTQ{6}{K} (Mixed) & $2.8$ GB & $75.0\%$ & $0.748$ & $3.54\%$ & $0.006$ & $0.084$ \\
        \qPTQ{4}{K} (Mixed) & $2.4$ GB & $74.8\%$ & $0.742$ & $5.39\%$ & $0.006$ & $0.083$ \\
        \qPTQ{4}{0} (Uniform) & $2.3$ GB & $74.3\%$ & $0.736$ & $5.66\%$ & $0.013$ & $0.133$ \\
        \qPTQ{2}{K} (Mixed) & $2.0$ GB & $70.0\%$ & $0.693$ & $12.19\%$ & $0.059$ & $0.263$ \\
        \midrule

        \multicolumn{7}{l}{\textbf{LLaMA-3.2-1B}} \\ \midrule
        Base (\BF) & $2.9$ GB & $\mathbf{73.2\%}$ & $\mathbf{0.728}$ & --- & --- & --- \\
        \qPTQ{8}{0} (Uniform) & $1.9$ GB & $73.0\%$ & $0.726$ & $\mathbf{1.25\%}$ & $\mathbf{0.000}$ & $\mathbf{0.013}$ \\
        \qPTQ{6}{K} (Mixed) & $1.6$ GB & $73.0\%$ & $0.726$ & $1.90\%$ & $0.002$ & $0.056$ \\
        \qPTQ{4}{K} (Mixed) & $1.4$ GB & $72.3\%$ & $0.720$ & $4.41\%$ & $0.003$ & $0.068$ \\
        \qPTQ{4}{0} (Uniform) & $1.4$ GB & $71.9\%$ & $0.716$ & $6.31\%$ & $0.047$ & $0.250$ \\
        \qPTQ{2}{K} (Mixed) & $1.2$ GB & $65.6\%$ & $0.654$ & $16.97\%$ & $0.058$ & $0.278$ \\
        \bottomrule
        \end{tabular}}
    \end{minipage}
\end{table*}

\begin{table*}[htbp]
    \centering

    \begin{minipage}{0.48\linewidth}
        \centering
        \caption{Complete Evaluation Metrics: \textbf{BoolQ}}
        \label{tab:app_boolq}
        \renewcommand{\arraystretch}{1.05}
        \resizebox{\linewidth}{!}{
        \begin{tabular}{@{}l|c|ccccc@{}}
        \toprule
        \textbf{Model \& Quantization} & \textbf{VRAM} & \textbf{Acc ($\uparrow$)} & \textbf{F1 ($\uparrow$)} & \textbf{\% Flips ($\downarrow$)} & \textbf{JSD ($\downarrow$)} & \textbf{TVD ($\downarrow$)} \\ \midrule

        \multicolumn{7}{l}{\textbf{Teuken-7B-base-v0.6}} \\ \midrule
        Base (\BF) & $15.0$ GB & $\mathbf{68.7\%}$ & $\mathbf{0.776}$ & --- & --- & --- \\
        \qPTQ{8}{0} (Uniform) & $10.4$ GB & $68.4\%$ & $0.775$ & $\mathbf{1.31\%}$ & $\mathbf{0.002}$ & $\mathbf{0.050}$ \\
        \qPTQ{6}{K} (Mixed) & $9.2$ GB & $68.4\%$ & $0.774$ & $2.20\%$ & $0.006$ & $0.087$ \\
        \qPTQ{4}{K} (Mixed) & $7.7$ GB & $67.8\%$ & $0.760$ & $7.34\%$ & $0.021$ & $0.152$ \\
        \qPTQ{4}{0} (Uniform) & $7.2$ GB & $67.3\%$ & $0.716$ & $25.23\%$ & $0.193$ & $0.510$ \\
        \qPTQ{2}{K} (Mixed) & $6.3$ GB & $61.1\%$ & $0.737$ & $19.05\%$ & $0.265$ & $0.589$ \\
        \midrule

        \multicolumn{7}{l}{\textbf{LLaMA-3.1-8B}} \\ \midrule
        Base (\BF) & $14.9$ GB & $44.3\%$ & $0.380$ & --- & --- & --- \\
        \qPTQ{8}{0} (Uniform) & $9.4$ GB & $44.6\%$ & $0.396$ & $\mathbf{2.32\%}$ & $\mathbf{0.001}$ & $\mathbf{0.027}$ \\
        \qPTQ{6}{K} (Mixed) & $7.7$ GB & $47.6\%$ & $0.482$ & $11.41\%$ & $0.002$ & $0.040$ \\
        \qPTQ{4}{K} (Mixed) & $6.3$ GB & $47.3\%$ & $0.473$ & $10.61\%$ & $0.003$ & $0.059$ \\
        \qPTQ{4}{0} (Uniform) & $6.0$ GB & $59.9\%$ & $0.735$ & $61.74\%$ & $0.042$ & $0.233$ \\
        \qPTQ{2}{K} (Mixed) & $4.8$ GB & $\mathbf{62.2\%}$ & $\mathbf{0.767}$ & $72.29\%$ & $0.051$ & $0.262$ \\
        \midrule

        \multicolumn{7}{l}{\textbf{Mistral-7B-v0.1}} \\ \midrule
        Base (\BF) & $14.0$ GB & $65.6\%$ & $0.747$ & --- & --- & --- \\
        \qPTQ{8}{0} (Uniform) & $8.1$ GB & $65.7\%$ & $0.748$ & $\mathbf{0.98\%}$ & $\mathbf{0.000}$ & $\mathbf{0.013}$ \\
        \qPTQ{6}{K} (Mixed) & $6.5$ GB & $\mathbf{65.8\%}$ & $0.762$ & $7.52\%$ & $0.000$ & $0.022$ \\
        \qPTQ{4}{K} (Mixed) & $5.0$ GB & $65.0\%$ & $0.764$ & $12.35\%$ & $0.000$ & $0.021$ \\
        \qPTQ{4}{0} (Uniform) & $4.8$ GB & $62.2\%$ & $\mathbf{0.767}$ & $25.96\%$ & $0.001$ & $0.029$ \\
        \qPTQ{2}{K} (Mixed) & $3.5$ GB & $62.2\%$ & $\mathbf{0.767}$ & $25.96\%$ & $0.002$ & $0.061$ \\
        \midrule

        \multicolumn{7}{l}{\textbf{Gemma-2B}} \\ \midrule
        Base (\BF) & $5.8$ GB & $\mathbf{62.2\%}$ & $\mathbf{0.767}$ & --- & --- & --- \\
        \qPTQ{8}{0} (Uniform) & $4.6$ GB & $\mathbf{62.2\%}$ & $\mathbf{0.767}$ & $\mathbf{0.00\%}$ & $\mathbf{0.002}$ & $\mathbf{0.045}$ \\
        \qPTQ{6}{K} (Mixed) & $4.0$ GB & $\mathbf{62.2\%}$ & $\mathbf{0.767}$ & $\mathbf{0.00\%}$ & $0.007$ & $0.098$ \\
        \qPTQ{4}{K} (Mixed) & $3.6$ GB & $\mathbf{62.2\%}$ & $\mathbf{0.767}$ & $\mathbf{0.00\%}$ & $0.007$ & $0.089$ \\
        \qPTQ{4}{0} (Uniform) & $3.5$ GB & $\mathbf{62.2\%}$ & $\mathbf{0.767}$ & $\mathbf{0.00\%}$ & $0.018$ & $0.164$ \\
        \qPTQ{2}{K} (Mixed) & $3.2$ GB & $\mathbf{62.2\%}$ & $\mathbf{0.767}$ & $\mathbf{0.00\%}$ & $0.065$ & $0.270$ \\
        \midrule

        \multicolumn{7}{l}{\textbf{LLaMA-3.2-1B}} \\ \midrule
        Base (\BF) & $3.1$ GB & $\mathbf{62.2\%}$ & $\mathbf{0.767}$ & --- & --- & --- \\
        \qPTQ{8}{0} (Uniform) & $2.5$ GB & $\mathbf{62.2\%}$ & $\mathbf{0.767}$ & $\mathbf{0.00\%}$ & $\mathbf{0.000}$ & $\mathbf{0.014}$ \\
        \qPTQ{6}{K} (Mixed) & $2.2$ GB & $\mathbf{62.2\%}$ & $\mathbf{0.767}$ & $\mathbf{0.00\%}$ & $0.001$ & $0.045$ \\
        \qPTQ{4}{K} (Mixed) & $2.0$ GB & $\mathbf{62.2\%}$ & $\mathbf{0.767}$ & $\mathbf{0.00\%}$ & $0.003$ & $0.064$ \\
        \qPTQ{4}{0} (Uniform) & $2.0$ GB & $\mathbf{62.2\%}$ & $\mathbf{0.767}$ & $\mathbf{0.00\%}$ & $0.048$ & $0.254$ \\
        \qPTQ{2}{K} (Mixed) & $1.8$ GB & $\mathbf{62.2\%}$ & $\mathbf{0.767}$ & $\mathbf{0.00\%}$ & $0.058$ & $0.280$ \\
        \bottomrule
        \end{tabular}}
    \end{minipage}
\hfill
    \begin{minipage}{0.48\linewidth}
        \centering
        \caption{Complete Evaluation Metrics: \textbf{CausalBench}}
        \label{tab:app_causalbench}
        \renewcommand{\arraystretch}{1.05}
        \resizebox{\linewidth}{!}{
        \begin{tabular}{@{}l|c|ccccc@{}}
        \toprule
        \textbf{Model \& Quantization} & \textbf{VRAM} & \textbf{Acc ($\uparrow$)} & \textbf{F1 ($\uparrow$)} & \textbf{\% Flips ($\downarrow$)} & \textbf{JSD ($\downarrow$)} & \textbf{TVD ($\downarrow$)} \\ \midrule

        \multicolumn{7}{l}{\textbf{Teuken-7B-base-v0.6}} \\ \midrule
        Base (\BF) & $15.0$ GB & $50.5\%$ & $\mathbf{0.628}$ & --- & --- & --- \\
        \qPTQ{8}{0} (Uniform) & $10.4$ GB & $50.0\%$ & $0.627$ & $\mathbf{3.27\%}$ & $\mathbf{0.002}$ & $\mathbf{0.048}$ \\
        \qPTQ{6}{K} (Mixed) & $9.2$ GB & $51.5\%$ & $0.619$ & $6.60\%$ & $0.006$ & $0.088$ \\
        \qPTQ{4}{K} (Mixed) & $7.7$ GB & $49.9\%$ & $0.608$ & $7.53\%$ & $0.021$ & $0.151$ \\
        \qPTQ{4}{0} (Uniform) & $7.2$ GB & $\mathbf{52.7\%}$ & $0.517$ & $35.13\%$ & $0.193$ & $0.511$ \\
        \qPTQ{2}{K} (Mixed) & $6.3$ GB & $45.9\%$ & $0.550$ & $26.27\%$ & $0.265$ & $0.590$ \\
        \midrule

        \multicolumn{7}{l}{\textbf{LLaMA-3.1-8B}} \\ \midrule
        Base (\BF) & $14.9$ GB & $\mathbf{50.1\%}$ & $\mathbf{0.000}$ & --- & --- & --- \\
        \qPTQ{8}{0} (Uniform) & $9.4$ GB & $\mathbf{50.1\%}$ & $\mathbf{0.000}$ & $\mathbf{0.00\%}$ & $\mathbf{0.001}$ & $\mathbf{0.027}$ \\
        \qPTQ{6}{K} (Mixed) & $7.7$ GB & $\mathbf{50.1\%}$ & $\mathbf{0.000}$ & $\mathbf{0.00\%}$ & $0.002$ & $0.039$ \\
        \qPTQ{4}{K} (Mixed) & $6.2$ GB & $\mathbf{50.1\%}$ & $\mathbf{0.000}$ & $\mathbf{0.00\%}$ & $0.003$ & $0.059$ \\
        \qPTQ{4}{0} (Uniform) & $6.0$ GB & $\mathbf{50.1\%}$ & $\mathbf{0.000}$ & $\mathbf{0.00\%}$ & $0.042$ & $0.233$ \\
        \qPTQ{2}{K} (Mixed) & $4.7$ GB & $\mathbf{50.1\%}$ & $\mathbf{0.000}$ & $\mathbf{0.00\%}$ & $0.051$ & $0.262$ \\
        \midrule

        \multicolumn{7}{l}{\textbf{Mistral-7B-v0.1}} \\ \midrule
        Base (\BF) & $14.0$ GB & $\mathbf{53.7\%}$ & $0.348$ & --- & --- & --- \\
        \qPTQ{8}{0} (Uniform) & $8.1$ GB & $53.7\%$ & $0.337$ & $\mathbf{1.27\%}$ & $\mathbf{0.000}$ & $\mathbf{0.012}$ \\
        \qPTQ{6}{K} (Mixed) & $6.5$ GB & $53.5\%$ & $0.274$ & $8.73\%$ & $0.000$ & $0.019$ \\
        \qPTQ{4}{K} (Mixed) & $5.0$ GB & $51.7\%$ & $0.376$ & $11.47\%$ & $0.000$ & $0.018$ \\
        \qPTQ{4}{0} (Uniform) & $4.8$ GB & $48.9\%$ & $0.661$ & $78.47\%$ & $0.001$ & $0.030$ \\
        \qPTQ{2}{K} (Mixed) & $3.5$ GB & $49.0\%$ & $\mathbf{0.662}$ & $78.47\%$ & $0.002$ & $0.059$ \\
        \midrule

        \multicolumn{7}{l}{\textbf{Gemma-2B}} \\ \midrule
        Base (\BF) & $5.7$ GB & $\mathbf{50.1\%}$ & $0.000$ & --- & --- & --- \\
        \qPTQ{8}{0} (Uniform) & $4.6$ GB & $\mathbf{50.1\%}$ & $0.000$ & $\mathbf{0.00\%}$ & $\mathbf{0.002}$ & $\mathbf{0.045}$ \\
        \qPTQ{6}{K} (Mixed) & $4.0$ GB & $\mathbf{50.1\%}$ & $0.000$ & $\mathbf{0.00\%}$ & $0.007$ & $0.098$ \\
        \qPTQ{4}{K} (Mixed) & $3.6$ GB & $\mathbf{50.1\%}$ & $0.000$ & $\mathbf{0.00\%}$ & $0.007$ & $0.089$ \\
        \qPTQ{4}{0} (Uniform) & $3.5$ GB & $48.9\%$ & $\mathbf{0.660}$ & $98.73\%$ & $0.018$ & $0.166$ \\
        \qPTQ{2}{K} (Mixed) & $3.1$ GB & $48.6\%$ & $0.057$ & $4.60\%$ & $0.066$ & $0.269$ \\
        \midrule

        \multicolumn{7}{l}{\textbf{LLaMA-3.2-1B}} \\ \midrule
        Base (\BF) & $3.0$ GB & $50.1\%$ & $0.000$ & --- & --- & --- \\
        \qPTQ{8}{0} (Uniform) & $2.5$ GB & $50.1\%$ & $0.000$ & $\mathbf{0.00\%}$ & $\mathbf{0.000}$ & $\mathbf{0.014}$ \\
        \qPTQ{6}{K} (Mixed) & $2.2$ GB & $50.1\%$ & $0.000$ & $\mathbf{0.00\%}$ & $0.001$ & $0.044$ \\
        \qPTQ{4}{K} (Mixed) & $2.0$ GB & $50.1\%$ & $0.000$ & $\mathbf{0.00\%}$ & $0.003$ & $0.063$ \\
        \qPTQ{4}{0} (Uniform) & $2.0$ GB & $50.1\%$ & $0.000$ & $\mathbf{0.00\%}$ & $0.049$ & $0.254$ \\
        \qPTQ{2}{K} (Mixed) & $1.8$ GB & $\mathbf{50.2\%}$ & $\mathbf{0.003}$ & $0.07\%$ & $0.058$ & $0.280$ \\
        \bottomrule
        \end{tabular}}
    \end{minipage}
\end{table*}

\vspace{12pt}

\end{document}